\documentclass[ba]{imsart}
\pubyear{}
\doi{}
\arxiv{}
\usepackage{amsthm}
\usepackage{amsmath}
\usepackage{natbib}
\usepackage[colorlinks,citecolor=blue,urlcolor=blue,filecolor=blue,backref=page]{hyperref}
\usepackage{graphicx}
\usepackage{graphicx,bm,hyperref,amssymb,amsmath,amsthm}
\usepackage{algorithmic,xcolor, float, url, graphicx}
\usepackage{graphicx}
\usepackage[caption=false]{subfig}

\usepackage{tikz}

\newcommand{\bi}{\begin{itemize}}
\newcommand{\ei}{\end{itemize}}
\newcommand{\ben}{\begin{enumerate}}
\newcommand{\een}{\end{enumerate}}
\newcommand{\be}{\begin{equation}}
\newcommand{\ee}{\end{equation}}
\newcommand{\bea}{\begin{eqnarray}} 
\newcommand{\eea}{\end{eqnarray}}
\newcommand{\ba}{\begin{align}} 
\newcommand{\ea}{\end{align}}
\newcommand{\bse}{\begin{subequations}} 
\newcommand{\ese}{\end{subequations}}
\newcommand{\bc}{\begin{center}}
\newcommand{\ec}{\end{center}}
\newcommand{\bfi}{\begin{figure}}
\newcommand{\efi}{\end{figure}}

\newcommand{\bmp}[1]{\begin{minipage}{#1}}
\newcommand{\emp}{\end{minipage}}
\newcommand{\bp}{\begin{proof}}
\newcommand{\ep}{\end{proof}}

\newcommand{\half}{\mbox{\small $\frac{1}{2}$}}

\newcommand{\R}{\mathbb{R}}

\newcommand{\bigO}{{\mathcal O}}

\newtheorem{thm}{Theorem}

\newtheorem{lem}[thm]{Lemma}

\newtheorem{pro}[thm]{Proposition}
\newtheorem{rmk}[thm]{Remark}
\newtheorem{dfn}[thm]{Definition}
\newcommand{\pushright}[1]{\ifmeasuring@#1\else\omit\hfill$\displaystyle#1$\fi\ignorespaces}

\newcommand{\al}{\alpha}
\newcommand{\eps}{\varepsilon}

\newcommand{\tal}{\widetilde\al}

\startlocaldefs
\endlocaldefs

\begin{document}

\begin{frontmatter}
\title{Delayed rejection Hamiltonian Monte Carlo for sampling multiscale distributions}
\runtitle{DRHMC for sampling multiscale distributions}

\begin{aug}
\author{\fnms{} \snm{}}
\author{\fnms{Chirag} \snm{Modi}\thanksref{addr1,addr2}\ead[label=e1]{Flatiron Institute}},
\author{\fnms{Alex} \snm{Barnett}\thanksref{addr1}},
\and
\author{\fnms{Bob} \snm{Carpenter}\thanksref{addr1}}%
% \ead[label=e3]{third@somewhere.com}%
% \ead[label=u1,url]{http://www.foo.com}}

\address[addr1]{Center for Computational Mathematics, Flatiron Institute, New York
    % \printead{e1} % print email address of "e1"
    % \printead*{e2}
}
\address[addr2]{Center for Computational Astrophysics, Flatiron Institute, New York
    % \printead{e1} % print email address of "e1"
    % \printead*{e2}
}

\runauthor{Modi et al.}

%\thankstext{<id>}{<text>}

\end{aug}

\begin{abstract}
The efficiency of Hamiltonian Monte Carlo (HMC) can suffer when sampling a distribution with a wide range of length scales, because the small step sizes needed for stability in high-curvature regions are inefficient elsewhere.
To address this we present a {\em delayed rejection} variant: if an initial HMC trajectory is rejected, we make one or more subsequent proposals each using a step size geometrically smaller than the last.  We extend the standard delayed rejection framework by allowing the probability of a retry to depend on the probability of accepting the previous proposal.
We test the scheme in several sampling tasks, including multiscale model distributions such as Neal's funnel, and statistical applications.
  Delayed rejection enables up to five-fold performance gains over optimally-tuned HMC, as measured by effective sample size per gradient evaluation.  Even for simpler distributions, delayed rejection provides increased robustness to step size misspecification.
  Along the way, we provide an accessible but rigorous review of detailed balance for HMC.
  \end{abstract}

%% ** Keywords **

\end{frontmatter}

%% ** Mainmatter **

%\section{}\label{}

% \begin{figure} 
% \includegraphics{<eps-file>}% place <eps-file> in ./img  subfolder
% \caption{}
% \label{}
% \end{figure}

% \begin{table} 
% *****************
% \begin{tabular}{lll}
% \end{tabular}
% *****************
% \caption{}
% \label{}
% \end{figure}

%%%%%%%%%%%%%%%%%%%%%%%%%%%%%%%%%%%%%%%%%%%%%%
%% Supplementary Material, if any, should   %%
%% be provided in {supplement} environment  %%
%% with title and short description.        %%
%%%%%%%%%%%%%%%%%%%%%%%%%%%%%%%%%%%%%%%%%%%%%%
%\begin{supplement}
%\stitle{???}
%\sdescription{???.}
%\end{supplement}

%% ** The bibliograhy **
%\bibliographystyle{ba}
%\bibliography{<bib-data-file>}% place <bib-data-file> in ./bib folder 

% ** Acknowledgements **
%\begin{acks}[Acknowledgments]
%\end{acks}

\section{Introduction}

Hamiltonian Monte Carlo (HMC), including auto-tuned extensions like the no U-turn sampler (NUTS), have become the de facto standard for high performance sampling of high-dimensional, differentiable distributions \citep{Duane1987, neal11, NUTS}.  
One reason for this is that HMC scales much better with dimension than
other Markov chain Monte Carlo (MCMC) methods such as
random-walk Metropolis or Gibbs sampling.
HMC's scalability derives from its ability to move large distances by approximating the Hamiltonian flow defined by the gradient of a distribution's log density function \citep{betancourt2017conceptual}.
As a result, HMC is believed to require $\mathcal{O}(d^{5/4})$ iterations to generate an independent sample in $d$ dimensions as compared to the $\mathcal{O}(d^2)$ samples required with random-walk Metropolis or Gibbs sampling \citep{neal11}.  The actual efficiency also depends strongly on geometric features of the density being sampled, particularly issues of high correlation between coordinates (leading to {\em stiffness}, i.e., ill-conditioning of the local curvature Hessian), and of spatially varying curvature (which defeats the use of global preconditioning to counteract stiffness).

One of the most common pathologies plaguing these algorithms
is the multiscale geometry of the posterior distributions \citep{Betancourt13, Pourzanjani19}: when the curvature of the log density varies spatially over a large dynamic range, small HMC time steps are needed for numerical stability in the high-curvature regions, preventing the use of the larger time steps needed for efficient sampling in smoother regions.
This geometry arises naturally in hierarchical models
that provide a population model for a group of effects in order to support regularization and partial pooling.  
However since all the contributions at the bottom of the hierarchy depend on the common global parameter, 
a small change in these high level parameters can induces large changes in the conditional density of the effects.
Consequently, when the data are sparse and inference is sensitive to the priors on these parameters,
the posterior density of these models looks like a ``funnel''
with a region of high density but low volume (``neck'') widening to a region of low density and high volume (``mouth'').
We show a two-dimensional example of this distribution in Figure \ref{fig:funnel}.
In the right panel of the same figure, we show the dramatic variations in condition number as the log scale parameter moves along the funnel.
Sampling this distribution is challenging because the mouth and the neck of the funnel contain equal probability mass
and so any sampling algorithm needs to handle these variations in curvature.

\begin{figure}[t!]
\centering
\includegraphics[width = 0.8\textwidth]{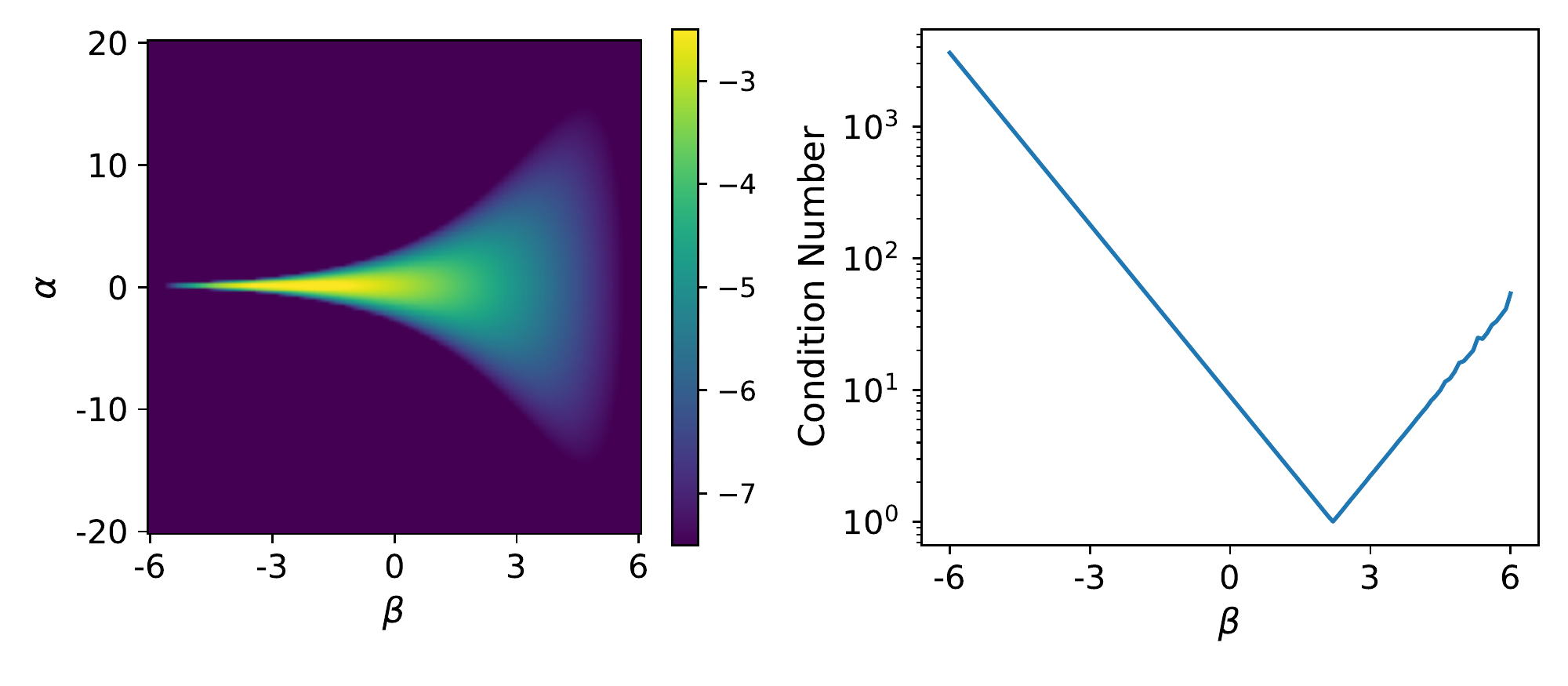}
\vspace{-20pt}
\caption{Neal's funnel. (Left) Natural log density of the two-dimensional funnel (Equation~\ref{eq:funnel})
showing a region of high density but low volume (``neck'', $\beta<0$) to the left of a region of low density and high volume (``mouth'', $\beta>0$).
(Right) Condition number of the inverse Hessian as a function of the log scale parameter $\beta$, along the slice $\alpha=0$.
}
\label{fig:funnel}
\end{figure}

A standard option for managing varying curvature is to use Hessian information.  This has led to the development of 
Riemannian HMC \citep{girolami2011riemann}, which follows a Riemannian metric based on the posterior curvature.  However, this is prohibitively expensive in high dimensions because it requires a positive-definite matrix at each point and many posteriors do not have positive-define Hessians.  One way to do this is to have an explicit form of the Fisher information matrix \citep{girolami2011riemann} or to use a conditioning operator like SoftAbs \citep{betancourt2013general}.

An alternative way to deal with high curvature is to approximate the Hamiltonian flow with an implicit symplectic integrator, which is able to naturally adjust stepping in different regions of phase space \citep{Pourzanjani19, brofos2021}.  Even simple implicit integration schemes like implicit midpoint are costly and present an algorithmic challenge for efficient and stable line search.  Ultimately, we believe it will be necessary to combine implicit integration and delayed rejection to achieve greater robustness in the face of even more challenging posterior sampling problems.

In this work, we develop an alternate approach inspired by the use of \textit{delayed rejection} (DR) methods to sample multiscale posterior distributions.
Recall that a high rejection rate increases autocorrelation of the Markov chain, reducing sampling efficiency.
Whenever a rejection would occur in the Metropolis algorithm,
DR methods do additional work, which can even exploit knowledge of the first rejection, to make a new proposal with a higher chance of acceptance \citep{haario2006dram}.
Since such new (possibly expensive) proposals are mostly made only when the standard proposal is poor, efficiency can be increased.
Although well studied in the context of random walk Metropolis-Hastings sampling
\citep{Mirathesis, Tierney99, Green01, haario2006dram}, there has been relatively little work done with these approaches for Hamiltonian Monte Carlo samplers \citep{Sohl-Dickstein14, Campos15}.

Previous approaches employing DR with HMC extend the same trajectory upon rejection so as to balance the additional cost by making larger jumps in the state space \citep{Sohl-Dickstein14,Campos15}.
However, this approach is not helpful if the chains are stuck in a region of high curvature where instability causes a high rejection rate.
In such cases,
as with delayed rejection in random-walk Metropolis \citep{Green01,haario2006dram},
it is more productive instead to change the proposal parameters with the goal of increasing the chance of acceptance.
In this work, we use this idea, building upon the original idea of delayed rejection \citep{Tierney99, Green01} to develop delayed rejection HMC (DRHMC).
Upon a rejection, DRHMC makes one or more subsequent proposals with smaller step sizes, with the aim that these are more likely to give stable leapfrog integration than their rejected predecessors. The result is a form of adaptivity with respect to step size,
a very successful idea in numerical integration more generally.\footnote{In the general delayed rejection method, the second and subsequent proposals may depend on the earlier proposals \citep{Green01}.}

In the rest of this paper, we begin by reviewing, in a mathematically rigorous yet accessible fashion, Metropolis Hastings, HMC, and delayed rejection methods in Section \ref{sec:background}.
With these tools, we then derive DRHMC in section \ref{sec:DRHMC} for one or more proposals. We show that DRHMC obeys detailed balance, discuss the cost of delayed proposals and outline probabilistic alternatives to reduce the cost of delayed rejection approaches.
Then in section \ref{sec:experiments}, we consider some toy models as well as actual data models,
and show that DRHMC can provide significant speed-ups as compared to traditional HMC in sampling tough multiscale distributions.
We also show that in cases with no such pathologies, probabilistic DRHMC is no more expensive than  HMC, 
thus suggesting its use as a robust alternative.
We conclude with discussion in section \ref{sec:discussion}.
Two short appendices contain proofs needed
in the main text.

\section{Metropolis-Hastings, delayed rejection, deterministic maps, and HMC}
\label{sec:background}

In this section we recap background material that is not easy to find gathered in an accessible format.
While being somewhat tutorial in nature, this
also sets up essential notation for the coming presentation of DRHMC.
We include a proof, avoiding technical measure theory notation, that HMC samples the correct target distribution, since in the literature this is usually presented either heuristically \citep{mackayerice,neal11,Sohl-Dickstein14,betancourt2017conceptual}, or rigorously but in highly abstract terms \citep{Andrieu20}.
%We leave some proof steps to the appendices.

In general we use $x\in S$ to denote the state in a continuous state space $S$, which can be taken as $\R^n$. (When we specialize later to HMC for sampling a target density over $\R^d$, we will set $S=\R^{2d}$.)
The goal of random-walk Metropolis, as with any MCMC method \citep{mackayerice,geyer11},
is to sample from a target probability density function (pdf) $\pi$ over $S$.
We assume that $\pi$ is absolutely continuous (AC), meaning that it can be represented
by a nonnegative function.
We assume the usual normalization $\int \pi(x) \, \textrm{d}x = 1$
(although all MCMC methods discussed can handle unnormalized $\pi$).
Unless indicated, all integrals are over $S$.

A Markov chain is defined by its {\em transition kernel} $k(x,y)$,
which gives the probability density function of transitioning
to the next state $y$, conditioned on the current state $x$.
The normalization is thus
\be
\int k(x,y) \, \textrm{d}y = 1, \qquad \forall x\in S~.
\label{knrm}
\ee
More formally, the transition kernel is a {\em measure} that depends on the parameter $x$,
and only when this measure is AC can it be written as a kernel function $k(x,y)$.
We refer the reader to \citep{steinanal,hunteranal,Billingsley,Andrieu20} for background on measure theory.
We will need to handle non-AC cases, but only for measures that can be described
using Dirac delta distributions, so will avoid technical language.
A necessary condition for MCMC to sample the correct pdf $\pi$ is its
invariance under the transition operator\footnote{Note that the operator acts from the right, the opposite convention from integral equations.},
\be
\int \pi(x) \, k(x,y) \, \textrm{d}x = \pi(y)~, \qquad \forall y\in S~.
\qquad \mbox{(Invariance)}
\label{eq:inv}
\ee
One way to ensure invariance is to construct kernels which maintain {\em detailed balance} (DB, also called ``reversibility''), meaning
\be
\pi(x) k(x,y) = \pi(y) k(y,x)~. %\qquad \forall x,y\in S~.
\qquad \mbox{(DB)}
\label{eq:db}
\ee
For AC kernels this simply means that the two sides are equal
for almost all $x,y\in S$.
For non-AC kernels the two sides may not be defined (e.g., infinite) for pairs $(x,y)$ of interest,
and one should interpret the left and right sides (once multiplied by $\, \textrm{d}x\textrm{d}y$) as measures over the Cartesian (tensor)
product space $S\times S$ (see, e.g., \cite[Ch.~18]{Billingsley}).
Then
\eqref{eq:db} should be interpreted as the left and right side being equal as product measures,
which means the {\em weak} sense
\be
\int_A \int_B \pi(x) k(x,y) \, \textrm{d}x \, \textrm{d}y =
\int_A \int_B \pi(y) k(y,x) \, \textrm{d}x \, \textrm{d}y~, \qquad \mbox{ for all (measurable) subsets } A,B \subset S.
\label{eq:dbm}
\ee
For ease of reading we will write statements of the form \eqref{eq:db} about
kernels over $(x,y)$ that represent product measures, with the understanding that they should be interpreted as in \eqref{eq:dbm}.

Finally, we recall the crucial fact that detailed balance implies invariance, which follows by substituting \eqref{eq:db} into \eqref{eq:inv} then using \eqref{knrm}.
\footnote{A sketch of the proof using the weak sense \eqref{eq:dbm} would be: choose $A=S$, swap the order of integration as justified by Fubini's theorem, and use \eqref{knrm}, leaving a weak statement of \eqref{eq:inv} for all $B\subset S$.}

\subsection{Metropolis-Hastings}
\label{s:mh}

The MH algorithm involves a {\em proposal} kernel $q(x,y)$ which gives, for each starting state $x$,
the (normalized) pdf over proposed states $y$.
For now we will assume that $q(x,\cdot)$ is AC for each $x\in S$.
%, which we assume is AC for now, and is also normalized so $\int q(x,y) \, \textrm{d}y = 1$ for all $x\in S$.
The proposal is accepted with some $x$- and $y$-dependent probability $\al(x,y)$, in which case the next state is $y$, otherwise the next state remains as $x$.
Thus the transition kernel $k(x,y)$ defining the resulting Markov chain
is, for each $x$, a mixture of the pdf $q(x,y)$ and the (rejected) point mass at $y=x$,
\be
k(x,y) = q(x,y)\al(x,y) + \delta_x(y) r(x)~,
\label{eq:mhker}
\ee
where $r(x)$ is the probability of rejection. Here $\delta$ is the Dirac delta distribution
defined in Euclidean space by $\delta(x) = 0$, $\forall  x\neq 0$, and
$\int \delta(x) \, \textrm{d}x = 1$, and we use the notation $\delta_x(y) = \delta(x-y)$.
Since the second term in \eqref{eq:mhker} is $x\leftrightarrow y$ symmetric whatever
the form of $r(x)$,
then for detailed balance \eqref{eq:db} to hold, we only need the condition on the first term
\be
\pi(x) q(x,y) \al(x,y) = \pi(y) q(y,x) \al(y,x)~.
\label{eq:alrat}
\ee
In the case where $q$ is not AC, then \eqref{eq:alrat} should be taken in the
sense of equality of product measures described above.
If $q$ is AC and everywhere positive then
the standard MH acceptance formula
\be
\al(x,y) = \min\left(\frac{\pi(y) q(y,x)}{\pi(x) q(x,y)}, 1 \right)
%~, \qquad x,y\in S~,
\label{eq:al}
\ee
is the most efficient%
\footnote{Here we mean efficiency in the sense that any other has higher probability of rejection.
This is simply because, for each $x,y\in S$,
either $\al(x,y)$ or $\al(y,x)$ is 1, the largest
allowed value for a probability.}
choice of $\al$ that satisfies \eqref{eq:alrat}.

\subsection{Delayed rejection for Metropolis-Hastings}
\label{s:dr}

Here we summarize standard delayed rejection as introduced in \citep{Mirathesis,Tierney99,Green01}.
The idea is make a second proposal with kernel $q_2(x,s,y)$ to $y$ if
the first proposal $q_1(x,s)$ from $x$ to $s$ is rejected (see Fig.~\ref{f:dr}(a)).
Note that $q_2$ may depend on both the current state $x$ and the rejected state $s$.
The transition kernel analogous to \eqref{eq:mhker} must account for three possible ways
to end up at state $y$: i) acceptance of $q_1(x,y)$, for which one uses
the usual MH probability $\al_1(x,y)$ obeying detailed balance \eqref{eq:alrat};
ii) acceptance of the second proposal, which occurs with some new probability
$\al_2(x,s,y)$; and iii) rejection of this second proposal.
For cases ii) and iii) one
must marginalize over all possible rejected first tries $s$.
Thus the transition kernel is
\be
k(x,y) = q_1(x,y)\al_1(x,y) +
\int q_1(x,s) [1-\al_1(x,s)][q_2(x,s,y)\al_2(x,s,y) + r_2(x,s)\delta_x(y)] ds
~,
\label{eq:dr}
\ee
where $r_2$ is the probability of rejection of the second proposal\footnote{As before, its form will be irrelevant because it lies on the diagonal $x=y$, so will not affect DB.}.
The factors $q_1(x,s) [1-\al_1(x,s)]$ in the integrand are the probabilities of making the first proposal ($q_1$) then rejecting it ($1-\al_1$).
The goal is then to choose $\al_2(x,s,y)$ such that DB is satisfied for the kernel $k$ given by \eqref{eq:dr}.
We have already established that this holds for the first term and the $r_2$ term,
which leaves only the middle $q_2$ term.
Substituting this middle term into the DB condition \eqref{eq:alrat} gives
\bea
\int \pi(x) q_1(x,s) [1-\al_1(x,s)]q_2(x,s,y)\al_2(x,s,y) \, ds \;\; = \nonumber \\
 \qquad \int \pi(y) q_1(y,s') [1-\al_1(y,s')]q_2(y,s',x)\al_2(y,s',x)\, ds'
\label{eq:dbmhdr}
\eea
where $s$ and $s'$ are (unrelated) dummy integration variables.
As before, if $q_2(x,s,\cdot)$ is AC, then this condition must hold for almost all $x,y\in S$.
% later we need to interpret in product measure sense
As most clearly explained by Mira \cite[Sec.~5.2]{Mirathesis}, one way (but not the only way) to enforce this condition is
simply to {\em set the integrands equal}.\footnote{Alternatively, one can assume that there exists a differentiable and invertible mapping from $(x, s, y)$ to $(y, s', x)$, and apply a change of variables to identify a more generic acceptance equation that is not constrained to follow the same path via $s$ from $y$ to $x$, as in \citep{Green01}.
}
This gives
\be
\pi(x) q_1(x,s) [1-\al_1(x,s)]q_2(x,s,y)\al_2(x,s,y) =
\pi(y) q_1(y,s) [1-\al_1(y,s)]q_2(y,s,x)\al_2(y,s,x), 
\label{eq:al2rat}
\ee
which now must hold for (almost) all $x,y,s\in S$.
Then, again assuming AC proposal pdfs with everywhere positive densities, the acceptance probability for the second proposal that maintains DB with the least rejection is \citep{Tierney99}
\be
\al_2(x,s,y) = \min\left( \frac{\pi(y) q_2(y,s,x) q_1(y,s) [1-\al_1(y,s)]}
  {\pi(x) q_2(x,s,y) q_1(x,s) [1-\al_1(x,s)]}
,1 \right)~,
\label{eq:almhdr}
\ee
Note that when we propose delayed rejection for maps coming from HMC in the next section,
we will not be able to use this method of Mira and Tierney,
and will need to return to the general integral condition \eqref{eq:dbmhdr}.
For a MH proposal, as compared to Eq. \ref{eq:al},
this acceptance probability has an extra factor $\frac{q_1(y,s) [1-\al_1(y,s)]}{q_1(x, s)[1-\al_1(x, s)]}$
which balances the probability of the first proposal being rejected at $y$ and $x$ respectively. 

\subsection{MH with deterministic proposals given by maps}

Metropolis-Hastings is usually presented assuming absolutely continuous proposal densities, so that equation
\eqref{eq:al} may be formulated.
However, HMC involves MH proposals that are given by deterministic maps,
which are not AC, rendering \eqref{eq:al} meaningless in this setting.
%This point seems not well appreciated in the literature.
Thus, in this section we derive rigorously the condition on the acceptance probability guaranteeing detailed balance, for the relevant class of maps.
By a map we mean
a smooth function $F: S\to S$, which thus takes each state $x$ to its image state $y=F(x)$.
The corresponding proposal kernel is
\begin{equation}
q_F(x,y) = \delta(y - F(x))~,
\label{eq:propmap}    
\end{equation}
which simply places the entire unit point mass at the point $y=F(x)$, hence is deterministic.
We will need the following two definitions.
\begin{dfn}[Involution]
  A map $F:S\to S$ is an involution if $F^{-1} = F$ as maps,
  that is, $F^2=I$ where $I$ is the identity map.
\end{dfn}
\begin{dfn}[Volume-preserving]
  A map $F:S\to S$ is volume (Lebesgue measure) preserving if
  $$
  \int_B \, \mathrm{d}x = \int_{F(B)} \, \mathrm{d}x~,\qquad \mbox{ for all (measurable) subsets } B \subset S~,
  $$
  where $F(B):=\{F(x): \, x\in B\}$ denotes the image of the set $B$.
\end{dfn}
If $DF(x) \in \R^{n\times n}$ is the Jacobian derivative matrix with
elements $(DF(x))_{ij} = \partial F_i(x) / \partial x_j $, $i,j = 1,\dots, n$,
then it is a standard result that volume preservation is equivalent to
$|\det DF(x)| = 1$ for all $x\in S$, i.e., a unit Jacobian determinant everywhere.
(See, e.g., \cite[Thm.~17.2]{Billingsley}.)

If MH is performed using deterministic proposals coming from maps
that are in both of the above special categories, then there is
a particularly simple condition that the acceptance probability should obey for DB
to hold, as follows. A simple proof is provided in Appendix \ref{app:deterministic}.
\begin{lem}[MH using a deterministic volume-preserving involution]
  Let $\pi$ be an AC target density.
  Let $F$ be a volume-preserving involution.
  Then MH with the deterministic proposal kernel $q_F$ given by \eqref{eq:propmap},
  with acceptance probability $\al$ obeying
  \be
  \pi(x) \al(x,y) = \pi(y) \al(y,x)   \qquad \forall x,y \in S
  \label{eq:alrats}
  \ee
  has detailed balance with respect to $\pi$, and therefore has $\pi$ as an invariant density.
  \label{lem:map}
    \end{lem}
The key point here is that the formula \eqref{eq:alrats} for the acceptance probability does not depend on the function $F$ at all, just on the ratio of target densities.
We have not found this well explained in the literature.
This will allow us in the following sections to place HMC, and our proposed DRHMC method, on a rigorous footing.

\begin{figure}[t]  
\includegraphics[width=\textwidth]{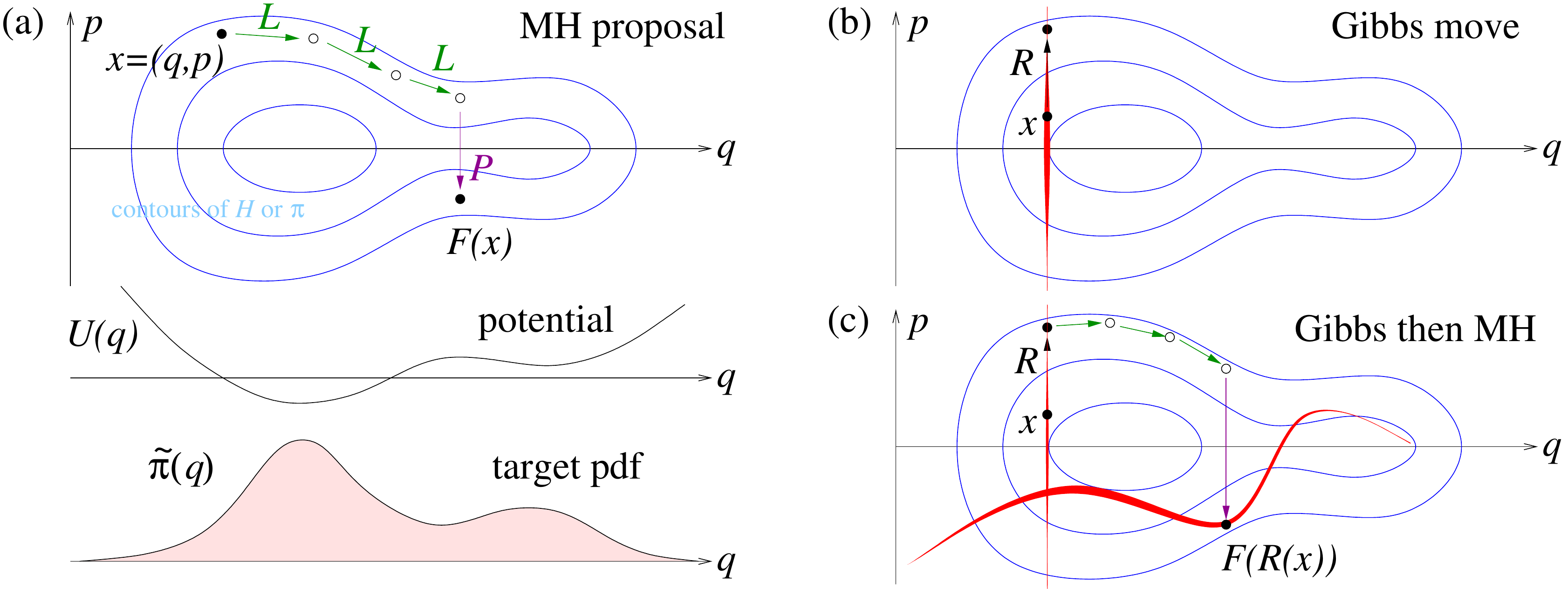}
\caption{Overview of HMC, sketched in $d=1$ dimensions. (a) shows the target density $\tilde\pi(q)$ (bottom), the
associated potential $U(q)$ (middle), and the resulting contours in 2D phase space $(q,p)$ of the Hamiltonian $H$
given by \eqref{H}. Each leapfrog step $L$ moves approximately along such a contour. For step 2 of HMC, the proposal move $F=L_\eps^nP$ is sketched, for $n=3$, where $P$ is the momentum flip. (b) shows the $p$ randomization (Gibbs move) in step 1 of HMC (red shows density of the kernel living on the $d$-dimensional slice $q = $ constant). (c) shows the composition of steps 1 and 2, comprising one HMC iteration (again red shows the resulting Markov kernel density, which lives on the union of a curved $d$-dimensional manifold and a constant-$q$ slice).
}\label{f:hmc}
\end{figure}

\subsection{Classical Hamiltonian Monte Carlo}
\label{s:hmc}

As our final piece of background, we review the Hamiltonian Monte Carlo (HMC) algorithm \citep{neal11}. 
We change the notation in this and the next section to overload $q$,
with $q\in \R^d$ now denoting the parameter vector of interest that is to be sampled.%
\footnote{Here we are following standard notation; we do not expect confusion to arise between $q$ as parameters, vs $q(\cdot, \cdot)$ as proposal function,
since the latter is always written as a function of {\em two} state points.}
The target pdf, which we call $\tilde\pi$, is assumed to be continuous and differentiable.
To draw samples $q$ from $\tilde\pi(q)$, HMC reinterprets the parameters of interest as
a position vector with associated potential energy function $U(q) = -\log \pi(q)$,
and simulates a Markov chain by approximating
the following Hamiltonian dynamics.
One introduces an auxiliary momentum vector $p\in \R^d$, which contributes a kinetic energy term $K(p) = \frac{1}{2}p^T M^{-1}p$, where $M$ is some symmetric positive definite mass matrix
that we take as fixed.
Then the Hamiltonian $H:\R^{2d} \to \R$ is the total energy function for the state $x := (q,p)$,
\be
H(x) = H(q,p) = U(q) + \frac{1}{2}p^T M^{-1} p~.
\label{H}
\ee
The state space $S = \R^{2d}$ is called {\em phase space};
see Fig.~\ref{f:hmc}(a) for an illustration.
Given initial data $x(0) = (q(0),p(0))$, the evolution of this physical system with respect to time $t$ is the first-order ODE system called Hamilton's equations,
\be
\left\{\begin{array}{lllll}\dot q &=& \nabla_p H(q(t),p(t)) &=& M^{-1} p(t) \\
\dot p &=& -\nabla_q H(q(t),p(t)) &=& -\nabla U(q(t))
\end{array}\right.
\label{flow}
\ee
where $\cdot = d/dt$ indicates the time derivative.
Intuitively, this motion is that of a point-mass ``rolling around'' in the potential well $U$,
in the absence of friction.
The force vector $-\nabla U$ attracts the ball so that it accelerates towards low-potential
(high-probability) regions.

HMC generates samples $x$ from the Gibbs pdf 
(also known as the Boltzmann or canonical distribution from statistical mechanics) defined by $H$, namely
\be
\pi(x) := Z^{-1} e^{-H(x)} = Z^{-1} e^{-U(q)} e^{-\half p^T M^{-1} p} = Z^{-1} \tilde\pi(q) e^{-\half p^T M^{-1} p}~,
\label{eq:Gibbs}
\ee
where $Z=\int_{\R^{2d}} H(x) \, \textrm{d}x = (2\pi)^{d/2}\sqrt{\det M}$ is the normalizing constant.
Note that, since $H$ was the sum of potential and kinetic terms, $q$ and $p$ are independent, with
the $q$-marginal of $\pi(x)$ being the target density $\tilde\pi(q)$.
Thus, given samples $x^{(i)}$ from $\pi$, by extracting their first $d$ coordinates one obtains samples from $\tilde\pi$.

HMC uses as its main step an MH step using a proposal from a particular deterministic
map $F$, which happens to
approximate Hamiltonian dynamics over a certain length of time $T$ followed by a negation
of the momentum.
The key property of this map---guaranteeing that is has the correct invariant
distribution $\pi$---will be that it is a volume-preserving involution;
the ancillary fact that it is an approximation to Hamiltonian dynamics is only relevant for creating a high {\em mixing rate} without excessive rejection in the MH acceptance step.
However, the exact dynamics is restricted to a
level set (energy shell) 
$H(q, p) = \textrm{constant}$ \cite[(2.13)]{neal11}, and staying permanently on this level set would {\em not} sample \eqref{eq:Gibbs}
correctly.
Thus, HMC alternates these Metropolis steps with a Gibbs sampling step that draws a fresh $p \sim \mathcal{N}(0, M)$.  This Gibbs update preserves the stationary distribution because the $p$ and $q$ terms factor in \eqref{eq:Gibbs}.
%The Gibbs update moves between energy shells, which is required for energy mixing.
Because the Metropolis and Gibbs updates both preserve the stationary distribution, so does their composition, which may be viewed as a single update in a Markov chain.

Let $L_\eps$ be the map that performs one leapfrog (Verlet) step with time step $\eps>0$.
Precisely, its action $(q',p') = L_\eps(q,p)$ is computed by the three sequential substeps,
\begin{align}
    \bar{p}     &\leftarrow\; p - \frac{\eps}{2} \nabla U(q)~,\nonumber \\
    q' &\leftarrow\; q + \eps M^{-1} \bar{p}~, \ \textrm{then} \nonumber \\
    p' &\leftarrow\; \bar{p} - \frac{\eps}{2} \nabla U(q') ~.
    \label{Leps}
\end{align}
The composition of $n = T/\eps$ such leapfrog steps is a $\bigO(\eps^2)$-accurate approximation to the exact dynamics \eqref{flow}
evolved to time $T$ (e.g. see \citep{neal11} for a derivation of the order of accuracy).
This composition is a volume-preserving involution, but does not conserve $H$ exactly.
Also we will need the ``momentum flip'' operator $P$ defined by $P(q,p) = (q,-p)$.

With these defined, a {\em single HMC iteration}
from the current state $x^{(i)}:=(q^{(i)},p^{(i)})$
comprises the two sequential steps: 
\begin{description}
    \item[Step 1. Gibbs sampling:] Resample the momentum $p^{(i)}$ from its Gaussian marginal distribution $p \sim \mathcal{N}(0, M)$, without changing $q^{(i)}$.\footnote{We use $\mathcal{N}(\mu, \Sigma)$ for normal distributions with location $\mu$ and covariance matrix $\Sigma$ and in the univariate case, $\mathcal{N}(\mu, \sigma^2)$ where $\sigma^2$ is the variance parameter.}
    This randomization step is shown as $R$ in Fig.~\ref{f:hmc}(b).
    (Note that there exist variants using partial randomization that we will not explore here
    \citep{neal11,Sohl-Dickstein14}.)
    
    \item[Step 2. Metropolis update:] Perform a Metropolis update on $x^{(i)}$, using a deterministic map
    $F = L_\eps^n P$ (here we compose operators to the right, so that $P$ is the final operator), where
    $n$ is a predetermined number of steps, $\eps>0$ is a time step, and
    $L_\eps$ and $P$ are the maps defined above. The proposal approximates Hamiltonian dynamics for time $T=n\eps$, followed by a $p$ flip, as sketched in Fig.~\ref{f:hmc}.
    Here, writing $x=x^{(i)}$ as the current state and $y=F(x)$ as the proposal, the step is accepted with probability
    $$
    \al(x,y) = \min\left(\frac{\pi(y)}{\pi(x)}, 1\right)  ~,
    $$
    being the most efficient rule satisfying \eqref{eq:alrats}.
    Upon acceptance $x^{(i+1)} \leftarrow y$, else $x^{(i+1)} \leftarrow x^{(i)}$.
 \end{description}
After each such iteration, $i$ is incremented, resulting in a Markov chain
$\{x^{(i)}\}_{i=0,1,\dots}$
from which expectations under $\tilde\pi$ may be estimated in the usual fashion \citep{geyer11}.

The following mathematical result, while covered recently using much more technical notation \citep{Andrieu20},
has a simple proof that we give in Appendix \ref{app:hmcproof}.
\begin{thm}[HMC has the correct invariant pdf]
  Let $\tilde\pi$ be a continuous, differentiable, positive pdf over $\R^d$,
  with associated Gibbs pdf $\pi$ over $\R^{2d}$ given by \eqref{eq:Gibbs}.
  The Markov chain with HMC update, given by the composition of
  steps 1 (Gibbs) and 2 (MH) defined above, has $\pi$ as an invariant pdf.
  \label{thm:hmc}
\end{thm}
In short, the proof is that step 1 (Gibbs) and step 2 (MH) each independently
preserve $\pi$ as an invariant pdf,
thus so does their composition.
In particular for step 2 this hinges on Lemma~\ref{lem:map} applied to $F=L_\eps^n P$;
its approximation of Hamiltonian dynamics is irrelevant for the proof.
It is also a common misunderstanding that their composition (the HMC iteration)
obeys detailed balance: although steps 1 and 2 separately do, their composition in general
does not.

It is worth pointing out that while first-order leapfrog integration ($L$) of Hamilton's equations is the most commonly used proposal in HMC, it is not the only choice.  The leapfrog integrator itself can be extended to higher orders \citep{creutz1989higher,yoshida1990construction}. 
Neal points out that a modified Euler step is valid \citep{neal11}, and recent works have proposed using other maps, such as implicit integrators \citep{Pourzanjani19,brofos2021evaluating} for multiscale distributions
or generalizing HMC with neural networks \citep{Levy17}.  
However, a lesson of the above is that approximating Hamiltonian dynamics is not necessary to have the correct invariant pdf; it is merely a convenient way to propose long-distance moves with high acceptance rates.

\begin{figure}[t!]
\centering \includegraphics[width=0.9\textwidth]{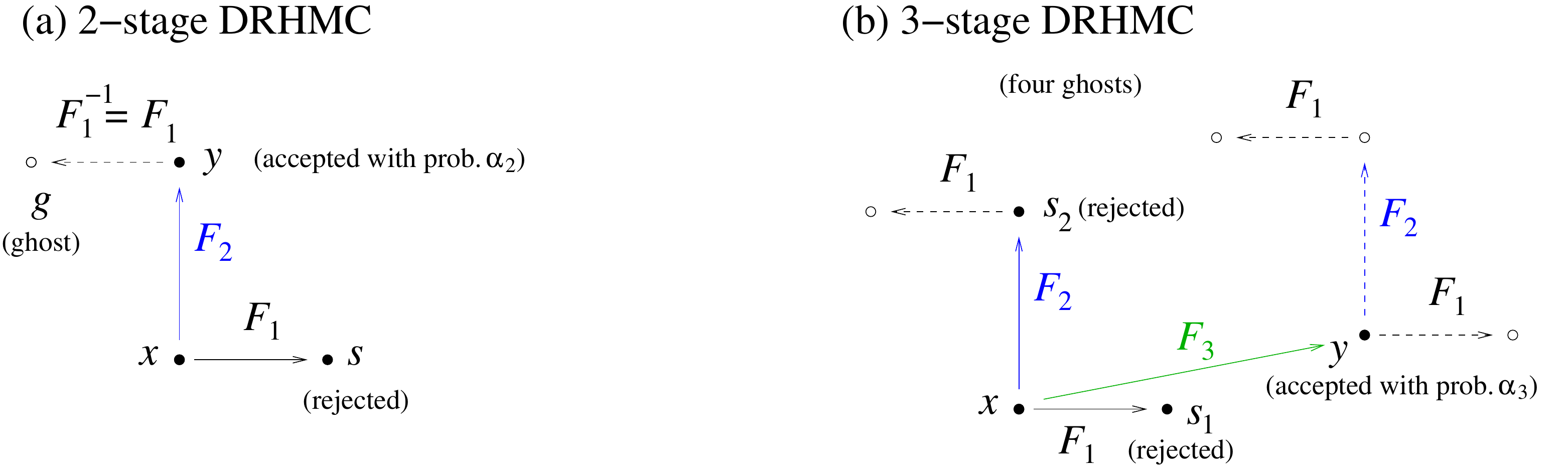}
\\
$\al_1(x,s)=\min\bigl( 1,\frac{\pi(s)}{\pi(x)} \bigr),
\;\;
\al_2(x,s,y)=\min\bigl( 1,\frac{\pi(s)}{\pi(x)}\frac{1-\alpha_1(y,g)}{1-\alpha_1(x,s)} \bigr);
\hfill \mbox{for $\alpha_3(x,s_1,s_2,y)$ see \eqref{eq:alhmcdr2}}
$
\caption{Sketch of states (in the
augmented state space $\R^{2d}$)
involved in delayed rejection HMC.
In this sketch, the locations are chosen purely to aid visualisation.
(a) Basic 2-stage scheme with proposal maps $F_1$ and $F_2$,
each of which is a certain number of leapfrog steps followed by a momentum-flip, hence an involution (see Section~\ref{sec:DRHMC}). The first map $F_1$ generates a proposal $s = F_1(x)$, which is accepted with probability $\alpha_1(x, s)$.  If the first proposal is rejected, the second proposal $y = F_2(x)$ is accepted with probability $\alpha_2(x, s, y)$.
The target density at the ``ghost'' $g=F_1(F_2(x))$ is also needed.
(b) 3-stage scheme, involving a third proposal map $F_3$
(see Section~\ref{sec:higherorder}).
The target density must be evaluated at $2^3=8$ states, four of which are ghosts (shown by open circles).
}
\label{f:dr}
\end{figure}

\section{Delayed rejection for HMC}
\label{sec:DRHMC}

Finally we have all the tools to combine delayed rejection (DR) with HMC.
We call the resulting algorithm DRHMC. 
As with classical HMC, we work with the extended state $x=(q,p)$ to sample from the desired distribution $\tilde\pi(q)$, which is the marginal of $\pi(x)$,
the resulting Gibbs pdf \eqref{eq:Gibbs} over the extended state.
As in Section~\ref{s:dr} we use $s$ to represent intermediate proposals in DR that have been rejected,
and $y$ will always represent the most recent proposal, i.e. the proposal made in the current DR stage. 

We keep the Gibbs step unchanged and apply DR only to the Metropolis step.
Consider $F_1 = L_\eps^n P$, a deterministic proposal map for some time step $\eps$ and
number of leapfrogs $n$.
The first acceptance probability remains the same as in classical HMC:
% HMC:
$\al_1(x,s) = \min \left[ \pi(s)/\pi(x), 1 \right]$.
If this first proposal with kernel $q_1(x,s) = \delta(s-F_1(x))$ gets rejected,
this suggests a possibility that $\pi(s)/\pi(x)$ is much less than 1, indicating very poor
approximate energy conservation so that $\eps$ was too large for stable integration.
This motivates a second proposal via a mapping $F_2$ which uses a smaller $\eps$.
The resulting second kernel is $q_2(x,s,y) = \delta(y-F_2(x))$, which is independent of $s$.

We now derive the detailed balance condition for $\al_2$ in a general setting.
We assume only that the maps $F_1$ and $F_2$ are volume-preserving involutions,
which is satisfied for HMC maps as discussed in the previous section.
%Then, since the first proposal is independent of the second proposal, 

For the second proposal, recall the general detailed balance condition \eqref{eq:dbmhdr} for delayed rejection. Substituting the above deterministic $q_1$ and $q_2$ kernels into this gives
\bea
\int_S \pi(x) \delta\big(s - F_1(x)\big) \, [1-\al_1(x,s)]\delta\big(y - F_2(x)\big)\al_2(x,s,y)  \mathrm{d}s \; = \quad\quad && \nonumber \\
 \int_S \pi(y) \delta\big(s' - F_1(y)\big) [1-\al_1(y,s')]\delta\big(x - F_2(y)\big)\al_2(y,s',x)  \mathrm{d}s'  \,.  &&\nonumber
\eea
However simply setting the integrands equal, as done by Mira and Tierney
to get \eqref{eq:al2rat}, 
fails here since the LHS delta selects $s=F_1(x)$ and the RHS delta selects $s'=F_1(y)$, but $F_1$ is injective so $s=s'$ could only hold if $x=y$.
Instead one {\em evaluates} the two integrals to get 
$$
\pi(x)[1-\al_1(x,F_1(x))]\delta(y-F_2(x)) \al_2(x,F_1(x),y) \; = \;
\pi(y)[1-\al_1(y,F_1(y))]\delta(x-F_2(y)) \al_2(y,F_1(y),x) ~,
$$
which must hold as kernels over $(x,y)$. Yet since $F_2=F_2^{-1}$, the two delta distributions
are the same, so
equality holds as singular measures living on the manifold $y=F_2(x)$ if $\al_2$ satisfies
\be
\pi(x)[1-\al_1(x,F_1(x))]\al_2(x,F_1(x),y) \; = \;
\pi(y)[1-\al_1(y,F_1(y))]\al_2(y,F_1(y),x) ~,  \quad \mbox{ for }\; y=F_2(x)~.
\label{eq:alrathmcdr}
\ee
To maximize the acceptance rate while obeying this constraint we set
\be
\al_2(x,F_1(x),y) \;=\; 
\min\!\left(1,  \frac{\pi(y)}{\pi(x)} \frac{1-\al_1\big(y,F_1(y)\big)}{1-\al_1\big(x,F_1(x)\big)} \right) ~,  \quad \mbox{ for }\; y=F_2(x)~.
\label{eq:alhmcdr}
\ee
Since all proposals are deterministic in DRHMC, it is now useful to simplify notation by
folding the known image points into the acceptance probabilities,
\bea
\tal_1(x) &:=& \al_1(x,F_1(x))
% =  \min \left[ \pi(F_1(x))/\pi(x), 1 \right]$          maybe include this, since always true
\label{ta1}
\\
\tal_2(x) &:=& \al_2(x,F_1(x),F_2(x))
\label{ta2}
\eea
This allows us to write the second acceptance probability obeying detailed balance as
\be
\tal_2(x) \;=\; 
\min\!\left(1,  \frac{\pi(F_2(x))}{\pi(x)} \, \frac{1-\tal_1(F_2(x))}{1-\tal_1(x)} \right)~.
\label{eq:talhmcdr}
\ee
Note that this is the same as the acceptance relation \eqref{eq:almhdr} from plain DR in the Metropolis case,
but setting all the proposal densities $q$ to unity.
However, we emphasise that its derivation is quite different, requiring care with deterministic maps, and relying on them being volume-preserving involutions.
In addition to the initial point of the trajectory $x$ and the two proposals $F_1(x)$ and $F_2(x)$, 
this rule demands, via $\tal_1(F_2(x))$, the pdf at a fourth state $g=F_1(F_2(x))$.
This is the first proposal that would have been made in a hypothetical chain,
had we started the chain in the reverse direction i.e. starting from $y$ to go to $x$.
Hence we call it a \textit{ghost preimage} of the second proposal. See Fig.~\ref{f:dr}(a).
While it is never proposed in the forward direction,
maintaining DB requires us to evaluate the density at this point. 

Finally, we describe the form of the new proposals that we test.
We consider delayed rejections
which reduce the step size of the leapfrog integrator by a constant {\em adaptivity factor} $a>1$
but maintain the same trajectory length or time of integration ($T$).
Hence we will propose $F_2 = L_{\eps/a}^{an} P$.
In Section~\ref{sec:experiments} we will show that this allows us to explore
regions in the phase space that otherwise face persistent rejections with classical HMC.
This completes the simplest form of DRHMC; however, we
find that the higher-order proposals described next can also help.

\subsection{Higher order proposals}
\label{sec:higherorder}

The previous section focused on making a second proposal when the first proposal in HMC gets rejected. 
The same formalism can be extended to allow a third proposal upon rejection of the second, a fourth upon rejection of the third, and so on.  In this section, we explicitly derive the acceptance probability for the third proposal in DRHMC
and give a general recursive relation for $k^{th}$~proposal. 
Mira \citep{Mirathesis} presents similar acceptance probabilities for higher-order delayed proposals in the Metropolis-Hastings case.
We also discuss the growth of the cost with number of proposals, since this determines the
trade-off with increased acceptance rate of DRHMC. 

If we reject the first two proposals starting from a state $x\in S$, namely $s_1=F_1(x)$ and $s_2=F_2(x)$,
we make a third proposal via a map $F_3$.
The resulting proposal kernel is $q_3(x,s_1,s_2,y) = \delta(y-F_3(x))$.

In this case, the transition kernel analogous to \eqref{eq:dr} must account for
four possible ways to end up at a state $y$: accepting the i) first, ii) second or the iii) third proposal
with their respective acceptance probabilities or iv) rejecting all and maintaining the current state. 
We have established the acceptance probabilities of case i) and ii) in the previous section.
For cases iii) and iv), the transition kernel must now marginalize over
all possible rejected first and second proposals $s_1, \, s_2$,
\bea
k(x,y) &=& q_1(x,y)\, \al_1(x,y) \\
&&\null + \int q_1(x,s_1) \, [1-\al_1(x,s_1)] \, q_2(x,s_1,y) \, \al_2(x,s_1,y) \, \textrm{d}s \\
&&\null + \int q_1(x,s_1) \, q_2(x,s_1,s_2) \, [1-\al_1(x,s_1)] \, [1-\al_2(x,s_1,s_2)] \\
\null  && \null \qquad \quad \null \times \left[q_3(x,s_1,s_2, y)\,\al_3(x,s_1,s_2,y) + r_3(x) \, \delta_x(y) \right]  \, \textrm{d}s_1 \, \textrm{d}s_2,
\label{eq:dr2}
\eea
where $r_3$ is the probability of rejecting the 3rd proposal.

As we saw in the previous section, since the first and second proposals are independent of the third proposal, 
their acceptance probabilities $\al_1$ and $\al_2$ are given by \eqref{eq:alrat} and \eqref{eq:alrathmcdr} respectively
and hence the first two terms of Eq. \eqref{eq:dr2} will maintain DB.
As with $r_2$, since $r_3$ also lies on the diagonal, it will also maintain DB regardless of its form.
Thus to maintain detailed balance for the third proposal, we only need the condition on $\alpha_3$,

\begin{align}
\int_S \int_{S} & \pi(x) \delta\big(s_1 - F_1(x)\big) [1-\tal_1(x)]  \delta\big(s_2 - F_2(x)\big)[1-\tal_2(x)] \delta\big(y - F_3(x)\big)\tal_3(x)\mathrm{d}s_1 \mathrm{d}s_2 =  \nonumber \\
&\int_S \int_{S}  \pi(y) \delta\big(s_1' - F_1(y)\big) [1-\tal_1(y)]  \delta\big(s'_2 - F_2(y)\big)[1-\tal_2(y)] \delta\big(x - F_3(y)\big) \tal_3(y)\mathrm{d}s'_1 \mathrm{d}s'_2   
\end{align}
where we have simplified the notation for acceptance probability via $\tal_3(x) := \al_3(x,F_1(x),F_2(x),F_3(x))$.
Then following the same steps as in the derivation of $\tal_2$ and evaluating the two integrals
allows us to write the the 3rd acceptance probability as 
\be
\tal_3(x) = \min\left[ \frac{\pi(y)  [1-\tal_1(y)]  [1-\tal_2(y)]}{\pi(x)  [1-\tal_1(x)] [1-\tal_2(x)]} , 1 \right]~.
\label{eq:alhmcdr2}
\ee

Continuing in this way, one can write down a recursive relation
for the acceptance probability of the $k${th} proposal obeying detailed balance,
\be
\tal_{k}(x) = \min\left[ \frac{\pi (y) \prod_{i=1}^{k-1}[1-\tal_i(y)]}{{\pi(x)  \prod_{i=1}^{k-1}[1-\tal_i(x)]}}, 1 \right]~,
\ee
where $y = F_k(x)$, and $\tal_{k}(x) := \al_{k}(x, F_1(x),...,F_{k}(x))$ is the notational shorthand. 

\subsubsection{Growth in cost with respect to the number of proposals}

It can be tempting to keep making successively higher order proposals to increase the acceptance rate, 
however %there are no free lunches and
it is important to be mindful of their increasing cost.
Thus we explicitly write down the full-form of $\tal_1$ and $\tal_2$ in the acceptance probability for the third proposal again to see the various ghost preimages that need to be evaluated.
Recall that 
\be
\tal_2(x) = \min\left[ \frac{\pi(F_2(x)) [1-\tal_1(F_2(x)))]}{{\pi(x)  [1-\tal_1(x)]}}, 1 \right] \nonumber
\ee
where
\be
\tal_1(x) = \min\left[ \frac{\pi(F_1(x))}{{\pi(x)}}, 1 \right]~. \nonumber
\ee

Substituting these forms in \eqref{eq:alhmcdr2}
we see that the denominator involves estimating the density at points
$x$, $F_1(x)$, $F_2(x)$ and $F_1(F_2(x))$
while the numerator requires the density at
$F_3(x)$, $F_1(F_3(x))$, $F_2(F_3(x))$, and $F_1(F_2(F_3(x)))$.
Of these $2^3=8$ points, only $F_1(x)$, $F_2(x)$ and $F_3(x)$ are proposals made in DRHMC
and the remaining states are the various ghost preimages that need to 
be evaluated to maintain detailed balance. 
Their form is sketched in Fig.~\ref{f:dr}(b).  
Evaluating the acceptance conditions for the $k$th proposal requires $2^k$ log density evaluations.  In our algorithm, computation is dominated by the number of gradient evaluations.\footnote{With automatic differentiation, the log density evaluations come for free with the gradient calculations.}

% Despite this apparent exponential growth in cost, the cost of DRHMC is only a {\em constant} factor larger than classical HMC run at the locally optimal step size.
% Consider the first DRHMC proposal $F_1$ with $n$ leapfrog steps of size $\eps$, i.e., a trajectory time of $T=n\eps$.
% If $\eps$ is greater than the largest stable step size for that trajectory,
% this first proposal will very likely be rejected, due to instability giving very poor $H$ conservation hence a tiny acceptance ratio.
% Our proposed higher-order DRHMC scheme then makes proposals with a sequence of step sizes $\eps/a, \eps/a^2, \dots$,
% so that the first that is likely to be accepted is the $k$th, 
% where $k$ is the smallest integer such that $\eps/a^{k-1}$ gives stable leapfrog integration.
% The total number of leapfrog steps needing for a $k$th order DRHMC proposal
% is $2^{k-1}n + 2^{k-2} an + \dots + a^{k-1} n$,
% where the first term is the for $F_1$ maps, the second for the $F_2$ maps, etc.
% (See Fig.~\ref{f:dr}(b) for the $k=3$ case.)
% This sum is $nka^{k-1}$ for $a=2$, or $\bigO(a^{k-1}n)$ for $a>2$.
% However, the final step size $\eps/a^{k-1}$ is within a factor of $a$
% of the optimal step size for HMC, so that
% even the optimal choice would need at least $a^{k-2}n$ steps.
% Thus for $a=2$, the DRHMC cost is only $\bigO(ak)$ more than
% HMC with the optimal step size, and for $a>2$, the DRHMC cost is $\bigO(a)$
% (independent of $k$) more than HMC.
Despite this apparent exponential growth in cost, the cost of DRHMC is only a {\em constant} factor larger than classical HMC run at the locally optimal step size.
For instance, consider a DRHMC setup where the first proposal is $F_1$ with $n$ leapfrog steps of size $\eps$, 
i.e., a trajectory time of $T=n\eps$,
and a sequence of step sizes $\eps/a, \eps/a^2, \dots$ for subsequent higher order proposals upon rejection. 
Let $k$ be the smallest integer such that $\eps/a^{k-1}$ gives stable leapfrog integration.
Then even for this optimal step size, classical HMC would need at least $a^{k-2}n$ steps.
On the other hand for DRHMC, since $\eps$ is greater than the largest stable step size for this trajectory,
this first proposal will very likely be rejected due to instability giving very poor $H$ conservation and hence a tiny acceptance ratio.
Our proposed higher-order DRHMC scheme then makes proposals with aforementioned sequence of step sizes $\eps/a, \eps/a^2, \dots$,
so that the first that is likely to be accepted is the $k$th with step size $\eps/a^{k-1}$.
The total number of leapfrog steps needed up to (and including) a $k$th order DRHMC proposal
is $2^{k-1}n + 2^{k-2} an + \dots + a^{k-1} n$,
where the first term is for the $F_1$ maps, the second for the $F_2$ maps, etc.
(See Fig.~\ref{f:dr}(b) for the $k=3$ case.)
This sum is $nka^{k-1}$ for $a=2$, or $\bigO(a^{k-1}n)$ for $a>2$.
Thus for $a=2$, the DRHMC cost is only $\bigO(ak)$ more than
HMC with the optimal step size, and for $a>2$, the DRHMC cost is $\bigO(a)$
(independent of $k$) more than HMC.

Thus we may summarize as follows.
\begin{rmk}
Although $k$th-order DRHMC has a cost per proposal that grows exponentially in $k$, the cost is only a constant factor more expensive than classical HMC proposals made with the ``correct'' (largest stable) step size.
\label{r:kcost}
\end{rmk}

\subsection{Probabilistic Delayed Rejection}
\label{sec:probdrhmc}

One way to reduce the average cost per iteration for DRHMC is to make the delayed rejections probabilistic and dependent on where we are in the distribution.
To motivate how this can be helpful, consider a case
when the cost of secondary proposal is much higher than the first proposal and even though
the first proposal function is well tuned for most of the state space, 
there are certain hard regions which can only be sampled by the second proposal. 
In this scenario, while we need DR to correctly sample the full distribution, we do not need it throughout the phase space.
Every time we make a secondary proposal upon getting a rejection in the good regions, we might not be trading excess cost with higher acceptance rate effectively. 
Thus instead of making the second and subsequent proposal
mandatory upon a rejection,
we would like to make them probabilistic such that we 
make a second proposal with probability $p_2(x, s)< 1$.
This modifies the second proposal kernel to $q_2(x,s,y) = p_2(x, s) \delta(y-F_2(x, s))$.
As was the case for the second proposal map $F_2$, this probability can also be informed by the previously rejected proposals in the same trajectory.
One can follow the steps from the previous section to maintain detailed balance and show that this modifies the acceptance probability as
\be
\tal_2(x) = \min\left[ \frac{\pi(y) [1-\tal_1(y)]  p_2\big(y, F_1(y)\big)}{\pi(x)  [1-\tal_1(x)]  p_2\big(x, F_1(x)\big)} , 1 \right]
\label{eq:alhmcdrp} 
\ee

Returning to the scenario outlined above, we see that one way to avoid secondary proposals in good regions is to construct a proposal probability that makes it less likely for a secondary proposal if the first proposal was rejected on random chance despite having high acceptance probability. 
On the other hand, if the first proposal was rejected strongly,
which might indicate that we are in a bad region of the state space for the first proposal, 
we make it  more likely to make a subsequent proposal with a new function. 
A simple heuristic proposal probability to achieve this is
% This can be achieved with the proposal probability
\be
% p_{j+1}(x, s) = \mathrm{max}\Bigg[0, 1 - \frac{\pi(s_j)}{\pi(x)}\Bigg]
p_{j+1}(x, s_j) = 1 - \alpha(x, s_j),
\label{eq:drprob}
\ee
where $s_j$ is the $j^{th}$ proposal made from the current position $x$, 
and $\alpha(s_j)$ is the acceptance probability of the last proposal.
Ideally however one would choose the proposal probability $p_{j+1}$ 
to maximize expected squared jump distance over effort for the next proposal.
% and Eq. \ref{eq:drprob} is just a simple heuristic that doesn't
% rely on estimating expected acceptance probabilities for higher order proposals. 
Detailed balance is maintained by including this factor $p_{j+1}$
in the acceptance condition $\alpha_{j+1}$ for the $j+1^{th}$ proposal along the lines of Eq. \ref{eq:alhmcdrp}.
In the experiments section, we will show how probabilistic delayed rejection can preserve the efficiency of basic HMC for simple distributions where HMC is effective.

\section{Experiments}
\label{sec:experiments}

In this section, we compare the performance of delayed rejection HMC (DRHMC) to that of standard HMC.

\subsection{Setup}
\label{sec:setup}
Given a current state $x$, HMC makes a proposal $y=F_1(x)$
where $F_1 = L_\eps^n P$ is the deterministic mapping that integrates Hamiltonian dynamics
with leapfrog integration for $n$ steps and step size $\eps$.
In DRHMC, we consider the first proposal to be the same as in HMC.
Upon rejection of the first proposal, we make $k-1$ subsequent proposals.
For each of these, we reduce the step size by a fixed factor $a>1$
while increasing the number of steps in proportion, to maintain a constant integration time.
This corresponds to a deterministic mapping,
\[
F_k(x) = L_{\eps a^{-(k-1)}}^{na^{k-1}} P(x).
\]

For every experiment and configuration, we run 50 chains with 1000 iterations for burn-in followed by 20,000 sampling iterations.

\subsubsection{Choice of parameters}
HMC has three tuning parameters, the step size $\eps$, the number of leapfrog steps $n$,
and the mass matrix $M$.  The total integration time is $T = n\eps$.

We use Stan \citep{Stan} to tune the reference values of these parameters using the following two steps.
\begin{enumerate}
    \item  We use the no-U-turn sampler (NUTS) \citep{NUTS} to select the integration time $T$.
NUTS is an adaptive algorithm that automatically stops every leapfrog trajectory when it starts to double back and retrace its steps and biases draws along the trajectory to later in the trajectory in an attempt to maximize expected squared jump distance.
Therefore, NUTS does not require tuning for $T$ during the warm-up phase. 
Following \citep{Wu2018}, we choose time of integration $T$ to be the 90th percentile of the trajectories followed by NUTS.\footnote{Unlike \citep{Wu2018}, we do not jitter the number of leapfrog steps.}
    \item After fixing $T$, we re-run Stan with HMC to estimate the optimal step size $\eps_{\mathrm{f}}$ and a diagonal mass metric, $M$.
\end{enumerate}

In addition to the HMC tuning parameters for integration time and step size, DRHMC has tuning parameters $k$ for the total number of 
of subsequent proposals made and
$a$ for the divisor by which step size is reduced for every subsequent proposal.
To develop an understanding of how these parameters impact the performance of DRHMC,
we report results for the grid of configurations with $k \in \{2, 3, 4\}$ and $a \in \{2, 5, 10\}$.

With its ability to reduce step sizes in subsequent proposals, DRHMC is more robust to the initial tuning of step size.  To demonstrate this, we evaluate HMC and RHMC with fixed and initial step sizes at, above and below the adapted step size, $\eps_0 =  0.5 \eps_{\mathrm{f}},\, \eps_{\mathrm{f}}, 2\, \eps_{\mathrm{f}}, 5\, \eps_{\mathrm{f}}$.

\subsubsection{Metric of comparison}
To measure sampling performance, we report the umber of log density and gradient evaluations required per effective draw, that is,
\begin{equation}
    \mathcal{C} = \frac{N_{\textrm{evals}}}{\mathrm{ESS}},
\label{eq:cost}
\end{equation}
where $N_{\textrm{evals}}$ is the total number of log density and gradient evaluations in the Markov chain, and $\mathrm{ESS}$ is the effective sample size for a parameter estimate extracted from the chain.  Log density and gradient evaluations dominate the cost of HMC, allowing us to ignore other implementation details.  Thus $\mathcal{C}$ is the inverse of efficiency; smaller $\mathcal{C}$ is better. Its value will depend on the expectation being evaluated, so we report results for posterior means of parameters $\theta$ and their squares $\theta^2$, the latter of which measures performance in estimating variance.

If $\rho_t \in (-1, 1)$ is the autocorrelation of a quantity in the Markov chain at lag $t$, the effective sample size is
\begin{equation}
    \mathrm{ESS} = \frac{N}{1 + 2 \sum_{t = 1}^{\infty} \rho_t}\ ,
\label{eq:essr}
\end{equation}
where $N$ is the total number of iterations \citep{geyer11}.  Standard errors for estimating parameters are then derived from the MCMC central limit theorem as as 
\[
\textrm{se} = \textrm{sd} / \sqrt{\textrm{ESS}}\ ,
\]
where $\textrm{sd}$ is posterior standard deviation.  The central limit theorem states that as effective sample size grows, errors approach a normal distribution,
\[
\widehat{\theta} - \theta
\sim \mathcal{N}(0, \textrm{se}^2)\ .
\]
This is usually a reasonable approximation even for modest effective sample sizes.  
Alternatively, if we know the true posterior mean value $\theta$, we can run independent Markov chains and calculate errors  $\widehat{\theta} - \theta$.  The sample standard deviation of the errors can be used to estimate $\textrm{se}$, from which we can back out effective sample size as
\[
\textrm{ESS} = \left(\frac{\textrm{sd}}{\textrm{se}}\right)^2.
\]

In the following experiments, depending on whether we have access to the true parameter distributions,
we will show results in terms of cost per effective sample calculated by autocorrelation length ($\mathcal{C}_{r}$) or estimated through errors in cases where posterior means and variances are known ($\mathcal{C}_c$).
In experiments with more than one parameter being sampled, we will show the cost for the parameter that mixes the slowest in the sense of having the lowest effective sample size.  We run multiple Markov chains and measure per-chain variation in cost by applying the bootstrap technique across chains.

\subsection{Neal's funnel}

\begin{figure}[t!]
\includegraphics[width = \textwidth]{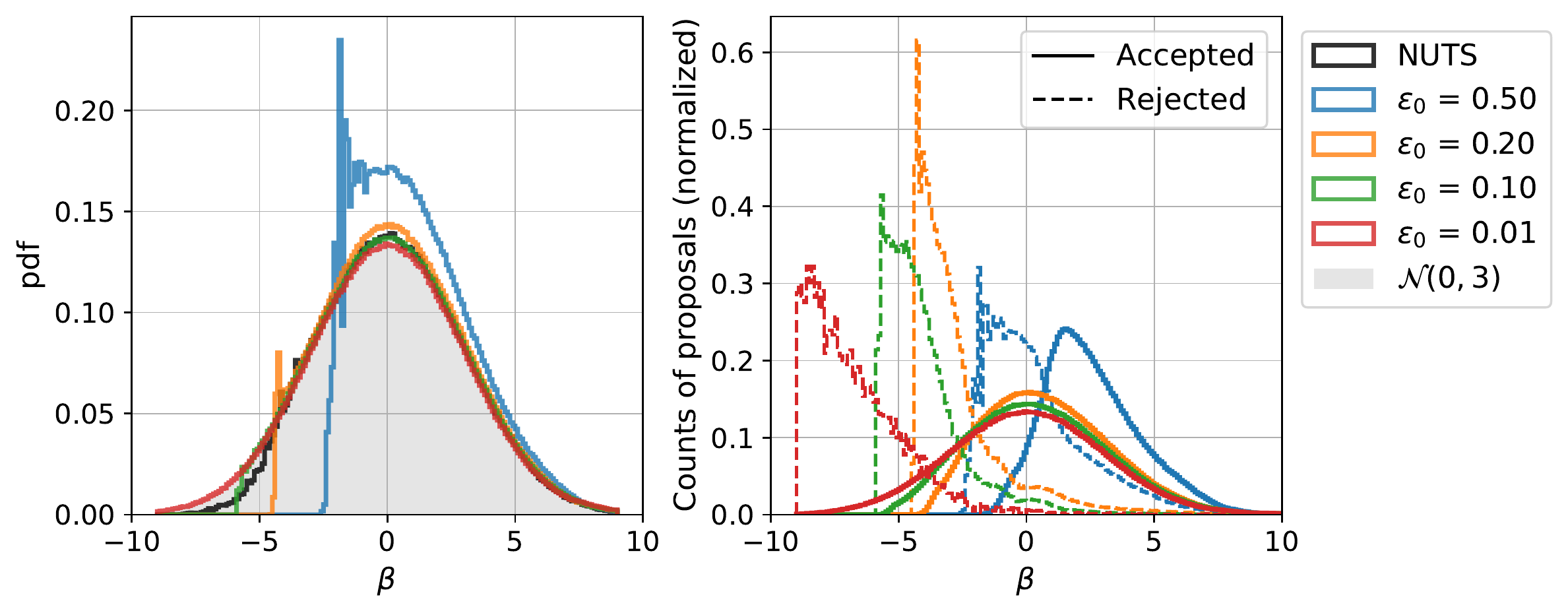}
\vspace{-20pt}
\caption{Difficulty of sampling Neal's funnel in $d=20$ dimensions. 
(Left) The marginal for $\beta$, which is  $\mathcal{N}(0, 3^2)$ as shown in gray, when the funnel is sampled with NUTS (default settings and HMC for different step sizes.
(Right) The fraction of accepted and rejected proposals for HMC as a function of $\beta$ for different step sizes.
All runs are done for 50000 samples and histograms are generated with bin-width of 0.1.}\label{fig:funnelhmc}
\end{figure}

We begin our experiments with the problem of sampling Neal's funnel,
upon which we touched in the introduction.
In $d$ dimensions, given a variance $\sigma^2$, 
we are interested in sampling $q=\{\beta, \alpha_2, \ldots, \alpha_d\}$ from Neal's funnel distribution \citep{Neal2000}, which is defined by
\begin{align}
\beta &\sim \mathcal{N}(0, \sigma^2) \nonumber \\
\alpha_i &\sim \mathcal{N}(0, e^{\beta}), \quad i = 2, 3, \dots, d,
\label{eq:funnel}
\end{align}
where as usual we use $\mathcal{N}(\mu,\sigma^2)$ to denote a normal pdf with mean $\mu$ and variance $\sigma^2$.
The resulting target pdf is
$$
\pi(\beta, \alpha_2, \dots, \alpha_d)
= \mathcal{N}(\beta \mid 0, \sigma^2)
  \prod_{i = 2}^d \mathcal{N}(\alpha_i \mid 0, \exp(\beta)).
$$  

Following Neal we set $\sigma=3$.
This distribution has equal probability mass in the regions $\beta<0$ and $\beta>0$.   
However the distribution has a wide range of length scales due to the curvature changing as $\beta$ ranges from large to small values (see Fig.~\ref{fig:funnel}).
This makes it challenging to sample the funnel efficiently with a constant step size. 
We can illustrate this with the help of Figure \ref{fig:funnelhmc},
which shows the empirical marginal of parameter $\beta$ when sampling the funnel in $d=20$ dimensions
with NUTS and HMC for different settings.
The correct marginal for $\beta$ is $\mathcal{N}(0, 3^2)$ which means that about $\sim 5\%$ of samples
should lie at $\beta<-5$.
However even for $\epsilon_0=0.2$, there are no samples in this regime. 
For NUTS, to push to $\beta<-5$, the step size had to be reduced such that $99\%$ of all proposals are accepted, 
as compared to $80\%$ default value of Stan and $65\%$ fraction considered optimal for normal distributions in HMC \citep{beskos2013optimal}.
To explore beyond the $3\sigma$ region, as required to get the right results for modest tail statistics, we need to reduce the step size further to $\eps=0.01$.

The extremely small step size necessary to explore the neck of the funnel
is very inefficient for exploring the mouth of the funnel.
As $\beta$ grows, the marginal $p(\alpha)$ approaches a lognormal distribution with $\sigma = 3$, and thus has long tails. 
The expected value of $\alpha^2$ is on the order of $10^2$, whereas the expectation of $\alpha^4$ is on the order of $10^8$. 
Due to the scale of the mouth of the funnel, the optimal step size is much larger than $\eps=0.01$ required to sample the neck of the funnel.  
Figure~\ref{fig:funnelhmc} provides an illustration of how well HMC can cover the mouth and neck of the funnel based on step size (left panel), as well as a comparison of the densities of accepted and rejected proposals for various step sizes (right panel).  All of the step sizes are able to sample the mouth of the funnel, however inefficiently, but in the neck of the funnel, acceptance rate dwindles to a sharp cutoff below which HMC is unable to sample.  To sample the tails of $\beta < 3 \cdot \textrm{sd}[\beta]$, we need to reduce step size even further below $\eps=0.01$. 

\begin{figure}[t!]
\centering
\includegraphics[width = 0.9\textwidth]{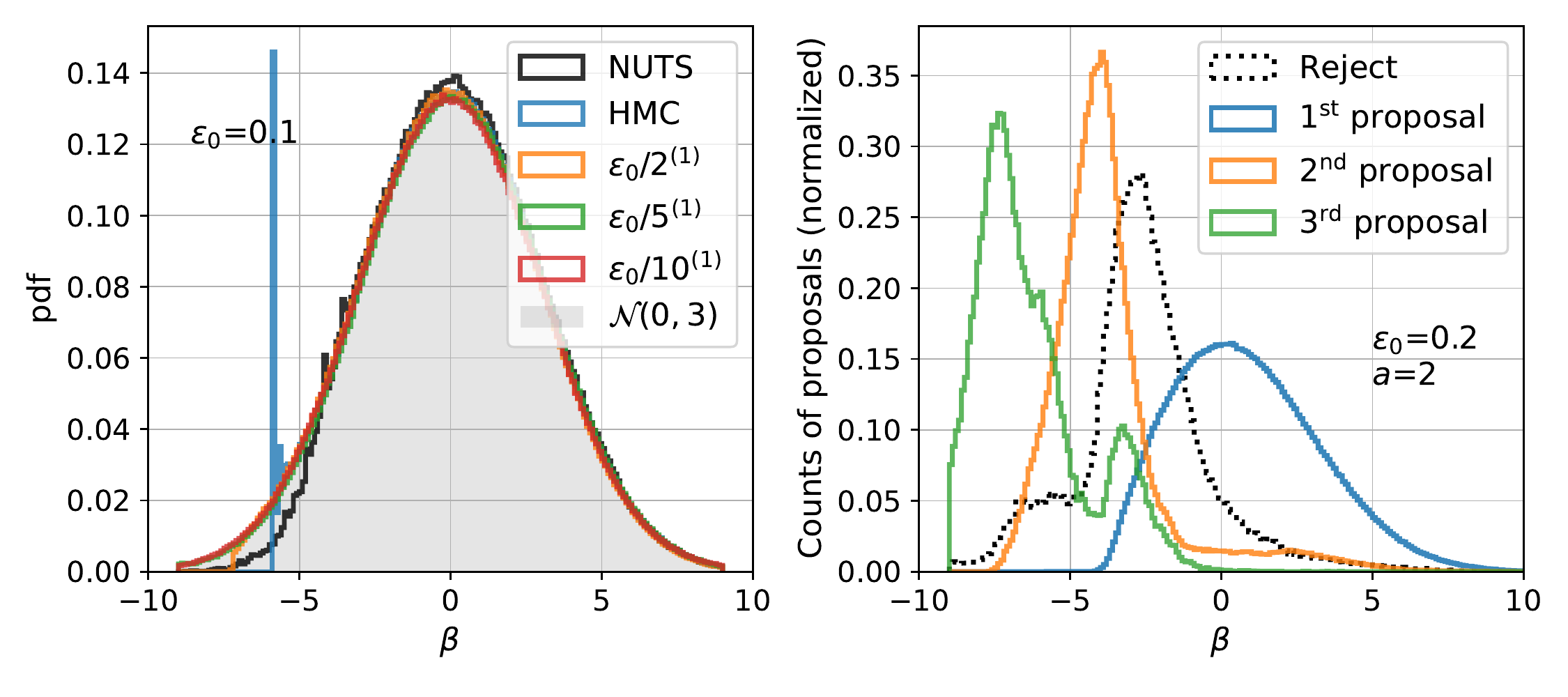}
\vspace{-10pt}
\caption{DRHMC for Neal's funnel in 20 dimensions. 
(Left) The marginal for $\beta$ sampled by NUTS, HMC and DRHMC with initial step size $\eps_0=0.2$
and second proposal with step size reduced by factors a=2, 5, 10.
% The legend indicates NUTS with default settings and  HMC for different step sizes.
(Right) The fraction of accepted and rejected proposals for  HMC as a function of $\beta$ 
in different stages of DRHMC when first step size $\epsilon_0=0.2$ and step size is reduced by factor $a=2$
at every stage. 
All runs are done for 50000 samples and histograms are generated with bin-width of 0.1.
}
\label{fig:funneldhmc}
\end{figure}

Figure \ref{fig:funneldhmc} shows how DRHMC can mitigate the sharp cutoff in the neck of the funnel by reducing step size as needed.  The left plot shows the marginal density $p(\beta)$ sampled with NUTS and HMC for step size $\eps_0=0.1$, as well as one-retry DRHMC with different stepsize reduction factors, $a$.
When the second proposal step size is reduced by factor of 10, DRHMC is able to sample $\beta$ to $\pm 3\sigma$.
The right panel shows the density of rejections and acceptances for the first, second, and third proposals of DRHMC with a larger step size ($e_0 = 0.2$), but allowing multiply retries with a reduction of $a = 2$.  With the possiblity of three proposals, an initial step size of $e_0 = 0.2$ is also able to sample $\beta$ to $\pm 3\sigma$.

\begin{figure}[t!]
\centering
\includegraphics[width = 0.7\textwidth]{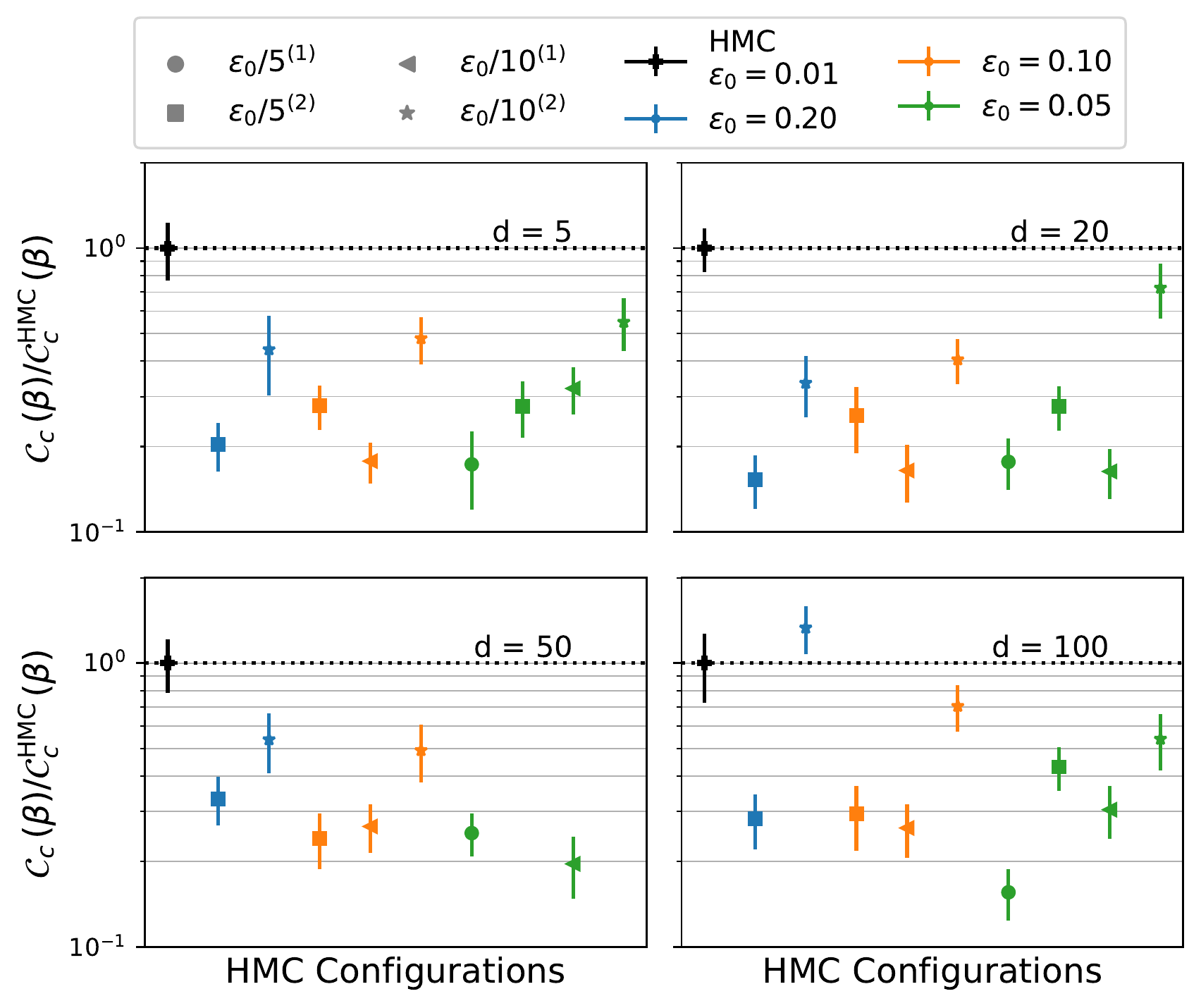}
\vspace{-10pt}
\caption{Efficiency of DRHMC for Neal's funnel.
These plots show the ratio of cost per effective sample of $\beta$ for DRHMC vs HMC
when sampling Neal's funnel for different dimensions ($d = 5, 20, 50, 100$).
ESS for the cost is estimated using standard error with reference samples (Eq.~\ref{eq:essr}).
Black points and horizontal dotted black lines show the reference ratio (=1 for  HMC with $\eps=0.01$).
Different colors show the cost for DRHMC with different first step sizes $\eps_0$; smaller is better.
Different shapes correspond to different configurations (number of proposals  $k$ and reduction factor $a$).
We show only configurations with $\eps_{\mathrm{min}} \leq 0.01$.
Error-bars are estimated by bootstrapping over chains.
}
\label{fig:funnelcost}
\end{figure}

Figure \ref{fig:funnelcost} illustrates the efficiency gain of DRHMC over HMC for Neal's funnel.
We show the cost per effective sample of $\beta$ for funnels of dimensions $d= 5, 10, 20, 100$.  $\textrm{ESS}_r$ can be biased to the high side because it only depends on the autocorrelation length of the chain and not that it is sampling the correct stationary distribution.  
Thus we use square errors to estimate effective sample size ($\mathrm{ESS}_c$).
This requires reference samples from the distribution, which are simple to generate independently using a non-centered parameterization of the funnel \citep{Betancourt13}.

Figure \ref{fig:funnelcost} shows that DRHMC is consistently a factor of 4 more efficient than HMC in terms of log density and gradient evaluations required for a given effective sample size; in some configurations the advantage is as much as a factor of 8.
% note you can't read 10 from Fig. 6; closer to 8.
We restrict attention to configurations for which HMC is able to sample $\beta$ to plus or minus three standard deviations (i.e., $\epsilon \leq  0.1$).  For all configurations, we applied Kolmogorov–Smirnov (KS) tests to the body and tails of the distributions to ensure we are sampling the correct distribution.

\subsection{Eight schools model}

One of the motivating applications for Bayesian hierarchical modeling was a meta-analysis of the effects of a test preparation intervention on students in eight schools \citep{Rubin81}.  The data consists of the differences in pre-test and post-test scores, which are reported as an average $y_n$ and standard deviation $\sigma_n$ for each school $n$.  The hierarchical model uses parameters
$\theta_n$ for the efficacy in each school and assigns them a hierarchical normal prior with unknown location $\mu$ and scale $\tau$. The generative model is as follows \citep{gelman2013bda}.\footnote{The positive half-Cauchy distribution uses a location-scale parameterization whereas the normal uses a location-variance parameterization.}
\[
\mu \sim \mathcal{N}(0,5^2), 
\qquad \tau \sim \mathrm{Cauchy}_+(0,5),
\]
\[
\theta_n \sim \mathcal{N}(\mu, \tau^2), \ \textrm{and}
\]
\[
y_n \sim \mathcal{N}(\theta_n,\sigma_n^2).
\]
The hyperparameter $\mu$ represents the average treatment effect across schools
and $\tau$ the scale of variation of effects among schools. 
As $\tau \rightarrow \infty$, the model approaches no pooling, i.e., each of the school treatment effects is estimated independently.  
As $\tau \rightarrow 0$, the model approaches complete pooling, i.e., all of the school treatment effects are the same.
For small values of $\tau$, the school-level effects $\theta_n$ are squeezed together; for large values, they are allowed to vary widely.  This yields a multiscale, funnel-like geometry in the $\tau$ and $\theta$ parameters, where we would expect delayed rejection to improve the performance of baseline HMC.

Figure~\ref{fig:eightschool} evaluates several configurations of the DRHMC algorithm as applied to the eight schools problem (see \ref{sec:setup}), plotting the cost of each configuration using the standard error method ($\mathcal{C}_c$) for the slowest mixing parameter.
We use the reference samples provided by the \texttt{posteriordb} database\footnote{\url{https://github.com/stan-dev/posteriordb}} 
to estimate the mean and variance of the parameters as needed to calculate error-based effective sample size ($\textrm{ESS}_c$). 
The best DRHMC configuration improves over the best HMC configuration by a factor of three for estimating the parameter mean. 
Different configurations for HMC perform the best for the first and second moment, with the cost of second moment estimation by DRHMC being on par with that of HMC.

\begin{figure}[t!]
\includegraphics[width = \textwidth]{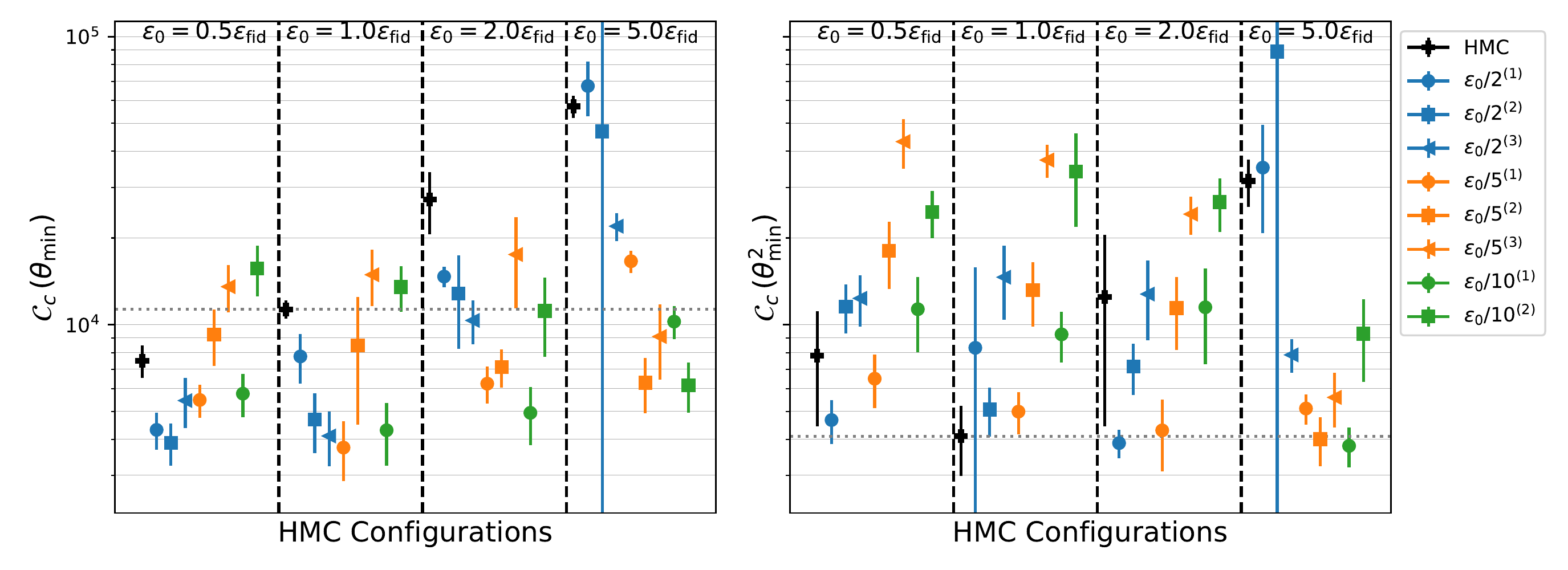}
\vspace{-10pt}
\caption{Cost per effective sample for HMC and DRHMC for the eight schools model. 
% \BC{Cost (vertical axis) should be explained in the figure.}
The two panels show the cost for the slowest dimension for the first ($\theta$) and second ($\theta^2$) moments respectively.
The cost for HMC is shown in black points, as estimated by using the standard error method for ESS.
Different configurations for DRHMC are shown in different colors (reduction factor, $a$) and shapes (number of proposals, $k$).
Different step sizes are separated by vertical dashed black lines. 
Dotted horizontal black line shows the cost for the reference configuration (as fit by Stan) of HMC.}
\label{fig:eightschool}
\end{figure}

\subsection{Gull's lighthouse}

Challenging posterior geometries arise even in simple two dimensional problems if the data is not very informative.  For example, consider estimating the direction of flashes emanating from a coastal lighthouse \cite[p.~59]{Gulls88}.  Assume the lighthouse is at position $x_0$ along a straight coast at distance $y$ into the sea.  Its light is spinning and emits a series of collimated flashes at random intervals which are then detected, each at a single point on the coastline.  Given N$_f$ flashes recorded at the positions $x_i, \, i=1,\dots,\mathrm{N}_f$, we perform a Bayesian estimation of the position of the lighthouse ($x_0, y$). 

The lighthouse flashes in a random direction $\theta$,
relative to vertical,
drawn from a uniform distribution on $(-\pi/2,\pi/2)$. Such a flash will be observed at location $x_i$ on the coast, where 
$\theta = \arctan((x_i - x_0)/y)$.  Applying a change of variables, the likelihood of observing a flash is
\begin{eqnarray*}
p(x_i \mid x_0, y) & = & \frac{y}{\pi (y^2 + (x_i-x_0)^2)}
\\[4pt]
& = & \textrm{Cauchy}(x_i \mid x_0, y).
\end{eqnarray*}
With improper uniform priors on $x_0$ and $y$, and the assumption that the flashes are independent, the posterior is proportional to the product of observation likelihoods,
\[
p(x_0, y \mid \{x_i\}) \ \propto \ \prod_i \textrm{Cauchy}(x_i \mid x_0, y).
\]
% \notes{what's the pathology of the distribution? It's still funnel, right, and not heavy tails of Cauchy distribution?}
% \BC{This is an important question.  If it has the heavy tails of the Cauchy distribution then the moments are ill defined and we should be evaluating with quantiles like medians.  We could evaluate a simple multivariate Student-t distribution 1, 2, 4, 8, and 16 d.o.f (from Cauchy to nearly normal).}

In Figure \ref{fig:gulls}, we show the cost $\mathcal{C}_r$ for the case with $N_f=3$ flashes
observed at $x_i = 0.9,\, 1.2,\, 1.21$, for both the parameters $x_0$ and $y$.
We estimate ESS by measuring autocorrelation length of the chains since there are no reference samples
available for this model.  For estimating $y$, whose effective sample size is an order of magnitude lower than that of $x_0$ with HMC, DRHMC is a factor of five more efficient;  there are no gains in sampling the parameter $x_0$ that mixes well with HMC. 

\begin{figure}[t!]
\includegraphics[width = \textwidth]{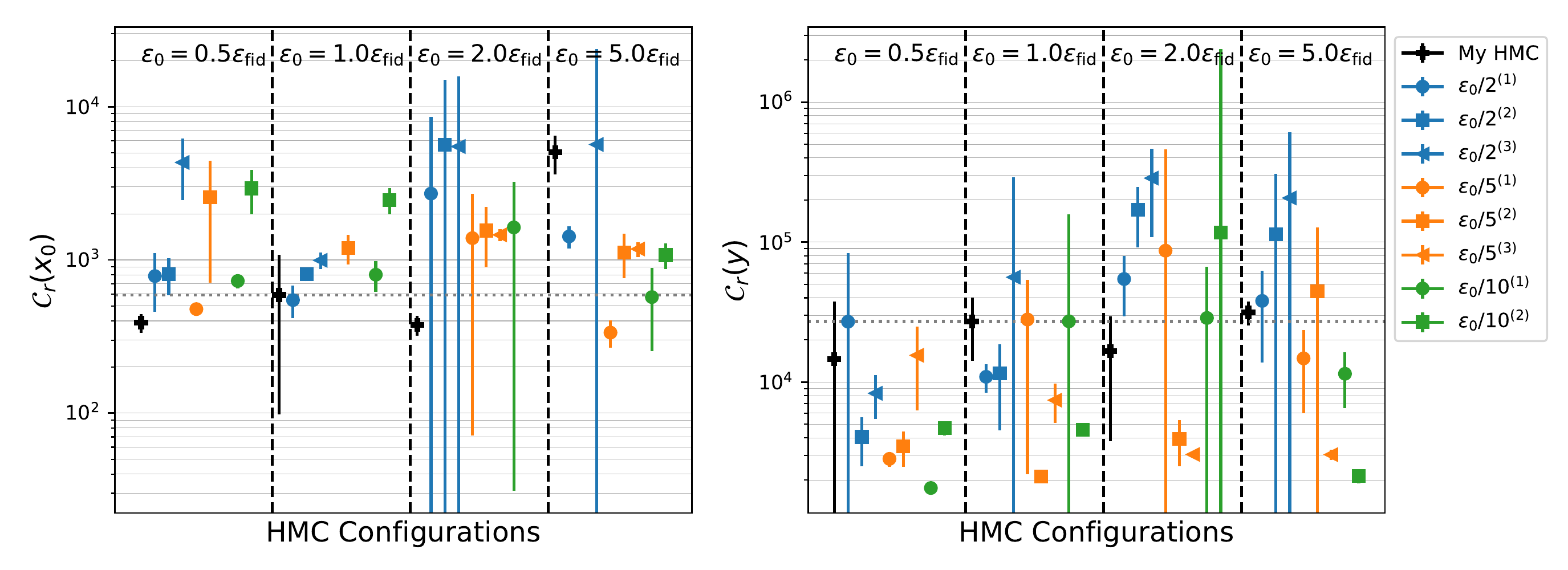}
\vspace{-10pt}
\caption{Cost per effective sample for HMC and DRHMC for Gull's lighthouse model.
The two panels show the cost for the two dimensions of the model. 
Symbol legends are the same as in Figure \ref{fig:eightschool}.
ESS is estimated by measuring autocorrelation length of the chains.
% \BC{I'm finding the through-line on these plots hard to follow.  I'm still thinking I'd prefer a line plot where the line collects the different configurations across the step sizes.  As is, I find myself trying to figure this out from the plots.}
}
\label{fig:gulls}
\end{figure}

\subsection{Gaussian Mixture Model}

Mixture models present problems for samplers with fixed step sizes when the mixture components are of different scales.  Multimodal distributions whose components have varying geometries also defeat global tuning for HMC.
HMC relies on tuning these parameters before sampling (for example, Stan first runs a warmup phase that performs adaptation before sampling begins).
As in other intrinsically multiscale problems, DRHMC has the potential to outperform baseline HMC by using different proposal scales
in different regions of the state space. 
 
To simulate the situation arising with multivariate posteriors, we consider a univariate Gaussian mixture with equal mixing weights on the components.  We take fairly separated locations $\mu_i$ that still allow mixing.  The scales $\sigma_i$ then vary by an order of magnitude. The model pdf is
\begin{equation}
p(\theta)  = \sum_{i=1, 2}\phi_i \cdot \mathcal{N}(\theta \mid \mu_i, \sigma_i^2),
\end{equation}
where we fix
\[
\phi_1= 0.5, \phi_2 = 0.5
\quad
\mu_1 = 0, \mu_2 = 3, \ \textrm{and}
\quad
\sigma_1 = 0.1, \sigma_2 = 1.
\]
Our goal is then to sample the univariate parameter $\theta \in \mathbb{R}$.  We choose this simple problem for illustration because sampling mixtures only becomes more challenging in higher dimensions with differently conditioned components, in situations where the modes are either more widely separated or more highly overlapping, or when the weights of the components are highly skewed.  The optimal step size for the components is directly proportional to the component's scale, which varies by an order of magnitude.

Figure~\ref{fig:gaussmix} shows that the best DRHMC configurations can be twice as efficient as HMC.  This gap can be made arbitrarily wide by increasing the number of dimensions and the difference in scales between the modes.  

\begin{figure}[t!]
\includegraphics[width = \textwidth]{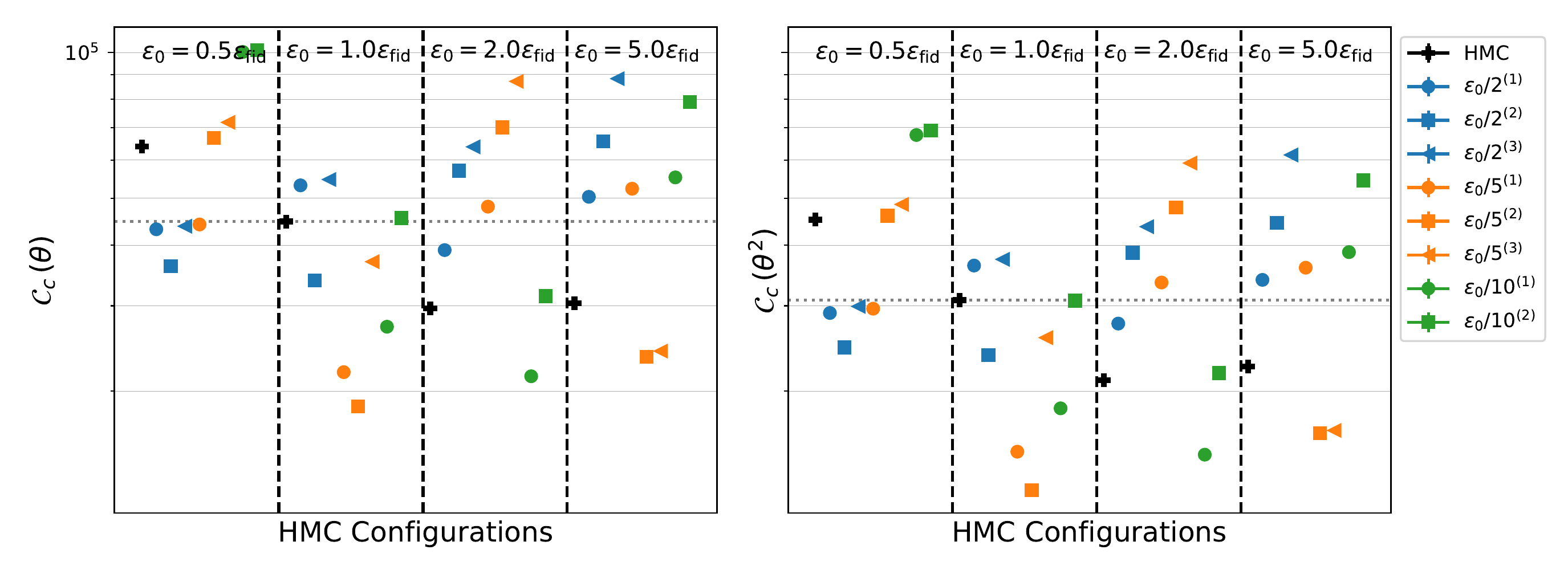}
\vspace{-10pt}
\caption{Cost per effective sample for HMC and DRHMC for the Gaussian mixture model with components varying in scale by a factor of ten.
The two panels show the cost for estimating the mean of $\theta$ and $\theta^2$, the second of which determines variance.  Symbol legends are the same as in Figure \ref{fig:eightschool}.  ESS is estimated with the standard error method after generating reference samples from the Gaussian mixture model.}\label{fig:gaussmix}
\end{figure}

\subsection{Stochastic volatility model}

\begin{figure}[t!] 
\subfloat[Delayed Rejection]{\includegraphics[width = \textwidth]{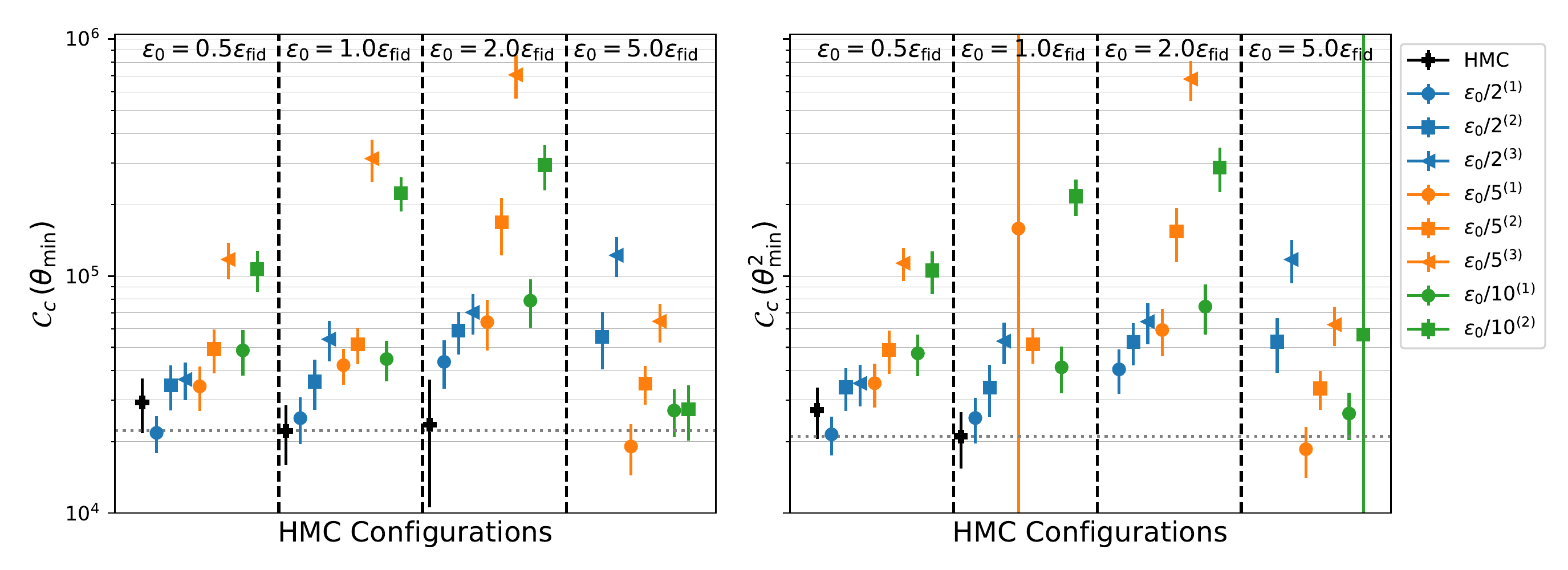}}
\vspace{-10pt}
\subfloat[Probabilistic Delayed Rejection]{\includegraphics[width = \textwidth]{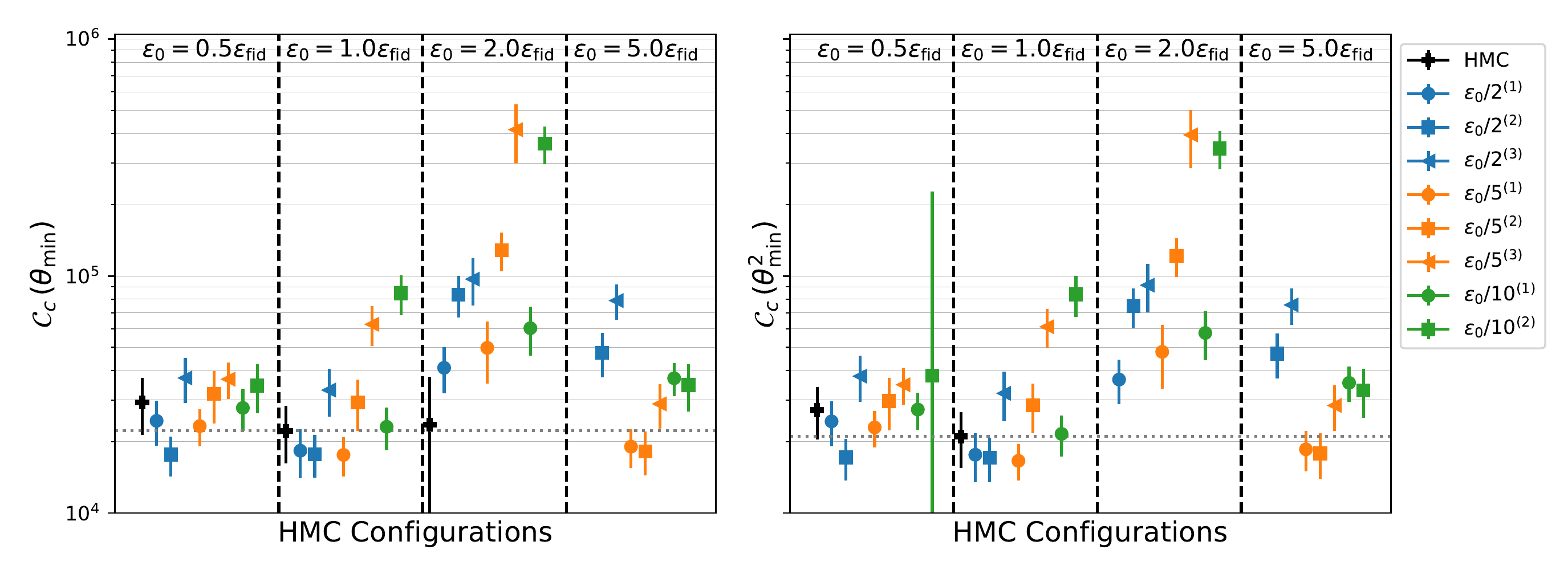}}
\vspace{-10pt}
\caption{Cost per effective sample for HMC and (a, top) DRHMC and (b, bottom) Probabilisitic DRHMC for Stochastic Volatility model.
The two panels show the cost for the slowest dimension for first ($\theta$) and second ($\theta^2$) moments respectively.
Symbol legends are the same as in Figure \ref{fig:eightschool}.
}
\label{fig:stochvcost}
\end{figure}

Finally we consider an example that does not suffer from the pathology of multiscale distributions,
but is still challenging due to high dimensionality and correlated parameters.
Stochastic volatility models \citep{Kim98} seek to model the volatility (i.e., variance) of a return on a financial asset,
such as an option to buy a security.
This changing volatility is modeled as a latent stochastic process in discrete time.
Given the mean corrected returns $y_t$ on an underlying asset at $T$ equally spaced time points as input data,
we are interested in estimating the latent parameter $h_t$ for the log volatility,
mean $\mu$ and variance $\sigma$ of log volatility, 
as well as the persistence of the volatility $\phi$.
Thus the parameter vector is $q=\{\mu, \sigma, \phi, h_{t=1,...,T}\}$, with
\begin{align}
\phi \sim \mathrm{uniform}(-1,1); \, \, \quad \sigma &\sim \mathrm{Cauchy}(0,5); \quad \mu \sim \mathrm{Cauchy}(0,10) \nonumber \\ 
h_1 \sim \mathcal{N}\Big(\mu, \frac{\sigma^2}{{1 - \phi^2}}\Big); \quad h_t &\sim \mathcal{N}\big(\mu + \phi (h_{t - 1} -  \mu), \sigma^2\big), \quad t = 2,3,\dots,T \nonumber \\
y_t &\sim \mathcal{N}\big(0, \mathrm{e}^{h_t}\big), \quad t = 1,2,\dots, T.   \nonumber
\end{align}
The posterior exhibits varying curvature due to the hierarchical prior on the volatility parameters, which reinforces the natural correlation among the volatility estimates due to their sequencing in time.  Figure \ref{fig:stochvcost}(a)
shows that the the additional computational cost of DRHMC 
ends up making it more costly per effective sample than HMC.

The cost of DRHMC is high her because the original step size is optimal, so that retrying with a lower step size only doubles computational costs.  To close the gap with fixed step size HMC when scales do not vary, we introduce a probabilistic modification of DRHMC that only retries when the previous proposal had a high probability of being rejected.  Specifically, We consider a scheme in which subsequent proposals are made with probability
\[
p(x, s_j) = 1 - \alpha(s_j),
\]
where $s_j$ is the previous proposal made from $x$ and $\alpha(s_j)$ is its acceptance probability.  Figure \ref{fig:stochvcost}(b) shows how retrying with a probability equal to the original failure chance avoids needless step size reduction, allowing DRHMC now to exceed slightly the efficiency of HMC.

\section{Discussion}
\label{sec:discussion}

We introduced a novel application of delayed rejection to Hamiltonian Monte Carlo (HMC) sampling in which subsequent proposals are made for the same integration time at a reduced step size.  We showed that in multiscale posteriors such as mixture models or hierarchical models, delayed rejection can boost performance by a factor of five or more.  We provided a proof that even if the initial step size is chosen to be too large, delayed rejection introduces at most a factor of two additional cost over choosing the optimal baseline step size. We also proved, in an accessible fashion avoiding measure theory, detailed balance for both classical Hamiltonian Monte Carlo and our new proposal.

In cases where the target density is not multiscale, we introduced a novel form of delayed rejection where retries are only attempted when the previous proposal had a high chance of failure.  Unlike the case for HMC, which will fail with potentially hard to diagnose biases, DRHMC with probabilistic retries is robust to the initial choice of step size, thus reducing overall costs when tuning step size is expensive.

In realistic problems, we often do not know if our target distribution suffers from multiscale or varying geometry pathologies as in the cases we considered.  For example, varying amounts of data and varying noise ratios in the data can dramatically change the posterior geometry, changing roughly normal posteriors into funnels or vice-versa, depending on the model parameterization \citep{papaspiliopoulos2007general,Betancourt13}.

Reducing the step size is not the only way of constructing delayed proposals.
Another approach would be to replace the leapfrog integrator altogether for retries, for example with an implicit symplectic integrator \citep{Pourzanjani19}. Such integrators may additionally be able to deal with stiffness arising from high correlation.  Delayed rejection HMC could also be combined with other improvements to HMC, e.g. ensemble preconditioning \citep{Matthews16}, Riemannian HMC \citep{Betancourt13}, and manifold HMC \citep{Au20}.

\bibliographystyle{ba}
%\bibliography{<bib-data-file>}% place <bib-data-file> in ./bib folder 
\bibliography{biblio}

\appendix
\section{Proof of Lemma~\ref{lem:map}}
\label{app:deterministic}

We first give a formal proof using the change of variables formula for integration,
then afterwards discuss an abbreviated version needing the transformation rule for the delta
distribution.

\begin{proof}
Consider Metropolis with kernel \eqref{eq:mhker} using as the proposal kernel $q$ the deterministic proposal $q_F(x,y)=\delta(y-F(x))$ from \eqref{eq:propmap} with $F$ a volume-preserving involution,
and an acceptance probability obeying \eqref{eq:alrats}.
The goal is to establish the weak form of DB \eqref{eq:dbm}.
As in Section~\ref{s:mh}, the rejection $r(x)$ component of $k(x,y)$ in \eqref{eq:mhker} is already symmetric, so that this may be dropped.
We are left to establish the weak form of \eqref{eq:alrat},
namely
\be
\int_A\int_B \pi(x)\al(x,y)q_F(x,y)\,\, \textrm{d}xdy \; = \; \int_A\int_B \pi(y)\al(y,x)q_F(y,x)\,\, \textrm{d}xdy 
\label{dbmF}
\ee
for all measurable subsets $A,B \subset S$.
We first substitute \eqref{eq:alrats} into the left-hand side,
then apply the sifting property of the delta distribution,
\bea
\int_A\int_B \pi(x)\al(x,y)\delta(y-F(x))\,\, \textrm{d}xdy &=&
\int_A\int_B \pi(y)\al(y,x)\delta(y-F(x))\,\, \textrm{d}xdy
\nonumber \\
&=&\int_{B\cap F^{-1}(A)} \pi(F(x)) \al(F(x),x) \, \, \textrm{d}x
\nonumber \\
&=&\int_{F(B)\cap A} \pi(y) \al(y,F^{-1}(y)) \cdot |\det DF(F^{-1}(y))|^{-1} dy
\nonumber \\
&=&\int_{F^{-1}(B)\cap A} \pi(y) \al(y,F(y))\, dy
\nonumber \\
&=&\int_A\int_B \pi(y)\al(y,x)\delta(x-F(y))\,\, \textrm{d}xdy
~.
\eea
Here in the 3rd line we applied the change of variables formula for integration
(e.g., \cite[Thm.~10.9]{babyrudin}) where $y=F(x)$, and in the 4th line used the
facts that $F^{-1}=F$ and that the Jacobian factor is everywhere unity.
The last equality used again the sifting property.
This verifies \eqref{eq:dbm}. Invariance of the pdf $\pi$ under the Markov
chain follows from DB as in the start of Section~\ref{sec:background}.
\end{proof}

A shorthand version of this proof may be instructive, and goes as follows.
The $x\leftrightarrow y$ symmetry of the expression $\pi(x)q_F(x,y)\al(x,y)$
needs to be verified.
Using \eqref{eq:alrats} leaves only the $x\leftrightarrow y$ symmetry of $q_F(x,y)$ to
be verified. The transformation rule
for the delta distribution under a nonsingular map (e.g., \cite[Sec.~2.4-5]{fariscalc}) gives
$$
q_F(x,y) = \delta(F(x)-y) = \sum_{z\in\R^n: F(z)=y} |\det DF(z)|^{-1} \delta(x-z) =
1 \cdot \delta(x-F^{-1}(y)) = \delta(x-F(y)) = q_F(y,x)~,
$$
where of course we again needed the unit Jacobian determinant,
and that $F$ is injective (giving only one preimage in the sum), and $F^{-1}=F$.

\section{Proof of invariance of HMC}
\label{app:hmcproof}

In this appendix we prove Theorem \ref{thm:hmc}.
We begin by first defining shear maps since our proof will build upon their volume-preserving property. We then prove an auxiliary lemma and conclude with the main proof.

\begin{dfn}[Shear]
  Any map on $\R^{2d}$ of the form
  $(q,p) \mapsto (q + G(p),p)$, or
  $(q,p) \mapsto (q, p +G(q))$,
  where $G:\R^d\to \R^d$ is some differentiable map, is called a {\em shear}.
\end{dfn}
\begin{pro}
  Any shear is volume-preserving.
  \label{p:shear}
\end{pro}
\begin{proof}
  Let $F$ be a shear.
  Computing its Jacobian with $d\times d$ blocks,
  $DF = \left[\begin{smallmatrix}I_d&DG\\0 &I_d \end{smallmatrix}\right]$, or
    $DF = \left[\begin{smallmatrix}I_d &0\\DG &I_d \end{smallmatrix}\right]$.
  In either case $\det DF \equiv 1$.
\end{proof}

\begin{lem}
  Let $n\in\{0,1,\dots\}$, and let $\eps>0$.
  Recalling the definitions of the leapfrog operator $L_\eps$ and momentum flip $P$ in Section~\ref{s:hmc}, the map $F=L_\eps^n P$ is a volume-preserving involution.
%   (Here, as above, we compose operators to the right, so this
%   means $L^n$ followed by $P$, although it happens not to
%   matter here.)
  \label{l:LnP}
\end{lem}
\begin{proof}
  $L_\eps$ is the composition of three steps, each of which is a shear and thus
  volume-preserving by Prop.~\ref{p:shear}. $P$ is obviously volume-preserving.
  Thus their composition $L_\eps^nP$ is volume-preserving.
  We now must prove the involution property.
  $L_\eps$ is time-reversible in the sense that, if $L(q_k,p_k)=(q_{k+1},p_{k+1})$,
  one may verify $L(q_{k+1},-p_{k+1})=(q_k,-p_k)$ by
  checking the three steps in reverse order ($\bar{p}$ is negated relative
  to its forward value).
  The same is true for $L_\eps^n$ for any $n$, by similarly reversing each leapfrog.
  Stating this algebraically, $PL_\eps^n P = L_\eps^{-n}$. Using $P=P^{-1}$ and
  rearranging, $(L_\eps^nP)^2 = I$, so $L_\eps^n P$ is an involution.
\end{proof}

\begin{proof}[Proof of Theorem \ref{thm:hmc}]
  It is sufficient to show that each step in the pair is $\pi$-invariant.
  This holds for step 1 (the Gibbs update of $p$) since it preserves the conditional
  over $p$, which is identical at each fixed $q$, while leaving $q$ unaffected.
  It holds for step 2 (one deterministic MH step) since by
  Lemma~\ref{l:LnP}, $F=L_\eps^n P$
  is a volume-preserving involution, so one can apply Lemma~\ref{lem:map}.
\end{proof}
\end{document}